%% file: main.tex
\documentclass[runningheads]{llncs}
\usepackage{graphicx}
\usepackage{comment}
\usepackage{amsmath, amssymb} 
\usepackage{color}
\usepackage{dirtytalk} 
\usepackage{xcolor}
\usepackage{verbatim} 
\usepackage{booktabs}
\usepackage{subcaption}
\usepackage{array,multirow}
\usepackage{authblk} 
\usepackage[colorlinks=true,linkcolor=red]{hyperref}
\usepackage{epstopdf}

\newcommand*\samethanks[1][\value{footnote}]{\footnotemark[#1]}

\usepackage[utf8]{inputenc}
\usepackage{mathtools}
\usepackage{units} 
\usepackage{pifont}
\newcommand{\cmark}{\ding{51}}%
\DeclarePairedDelimiter\abs{\lvert}{\rvert}%
\newcommand{\argmin}{\mathop{\mathrm{argmin}}}

\newcommand{\unde}[1]{\underline{#1}}

\usepackage[colorinlistoftodos,bordercolor=orange,backgroundcolor=orange!20,linecolor=orange,textsize=scriptsize]{todonotes}

\newcommand{\eg}{\textit{e.g.}}
\newcommand{\etal}{\textit{et al.}~}
\newcommand{\ie}{\textit{i.e.}~}

\begin{document}
\pagestyle{headings}
\mainmatter
\def\ECCVSubNumber{869}  

\title{Gabor Layers Enhance Network Robustness}

\titlerunning{Gabor Layers Enhance Network Robustness}
%
 \author{Juan C. Pérez \thanks{denotes equal contribution}\inst{1} \and
 Motasem Alfarra \samethanks\inst{2} \and
 Guillaume Jeanneret \samethanks\inst{1} \and
 Adel Bibi\inst{2} \and
 Ali Thabet\inst{2} \and
 Bernard Ghanem\inst{2} \and
 Pablo Arbeláez\inst{1}}
\authorrunning{J.C. Pérez et al.}
%
\institute{
Universidad de los Andes, Colombia \and 
King Abdullah University of Science and Technology (KAUST), Saudi Arabia
}
\maketitle

\input{sections/0abstract}
\input{sections/1introduction}
\input{sections/2related_work}
\input{sections/3methodology}
\input{sections/4experiments}

\input{sections/5conclusions}
\clearpage
\input{sections/6acknowledgments}

{\small
\bibliographystyle{splncs04}

\input{sections/main.bbl}
}


\clearpage
\input{sections/supplemental}

%
%

\end{document}

%% file: sections/0abstract.tex
\begin{abstract}
We revisit the benefits of merging classical vision concepts with deep learning models. In particular, we explore the effect on robustness against adversarial attacks of replacing the first layers of various deep architectures with Gabor layers, \ie convolutional layers with filters that are based on learnable Gabor parameters. We observe that architectures enhanced with Gabor layers gain a consistent boost in robustness over regular models and preserve high generalizing test performance, even though these layers come at a negligible increase in the number of parameters. We then exploit the closed form expression of Gabor filters to derive an expression for a Lipschitz constant of such filters, and harness this theoretical result to develop a regularizer we use during training to further enhance network robustness. We conduct extensive experiments with various architectures (LeNet, AlexNet, VGG16 and WideResNet) on several datasets (MNIST, SVHN, CIFAR10 and CIFAR100) and demonstrate large empirical robustness gains. Furthermore, we experimentally show how our regularizer provides consistent robustness improvements.
   
   
\end{abstract}

%% file: sections/1introduction.tex
\section{Introduction}
Deep neural networks (DNNs) have provided outstanding gains in performance in several fields, from computer vision \cite{alexnet,resnet_he} to machine learning \cite{deepleaning} and natural language processing \cite{speech}. However, despite this success, powerful DNNs are still highly susceptible to small perturbations in their input, known as adversarial attacks \cite{goodfellow2014explaining}. Their accuracy on standard benchmarks can be drastically reduced in the presence of perturbations that are imperceptible to the human eye. Furthermore, the construction of such perturbations is rather undemanding and, in some cases, as simple as performing a single gradient ascent step of a loss function with respect to the image \cite{goodfellow2014explaining}. 

\begin{figure}[t]
    \centering
    \begin{subfigure}{.46\textwidth}
        \includegraphics[width=\textwidth]{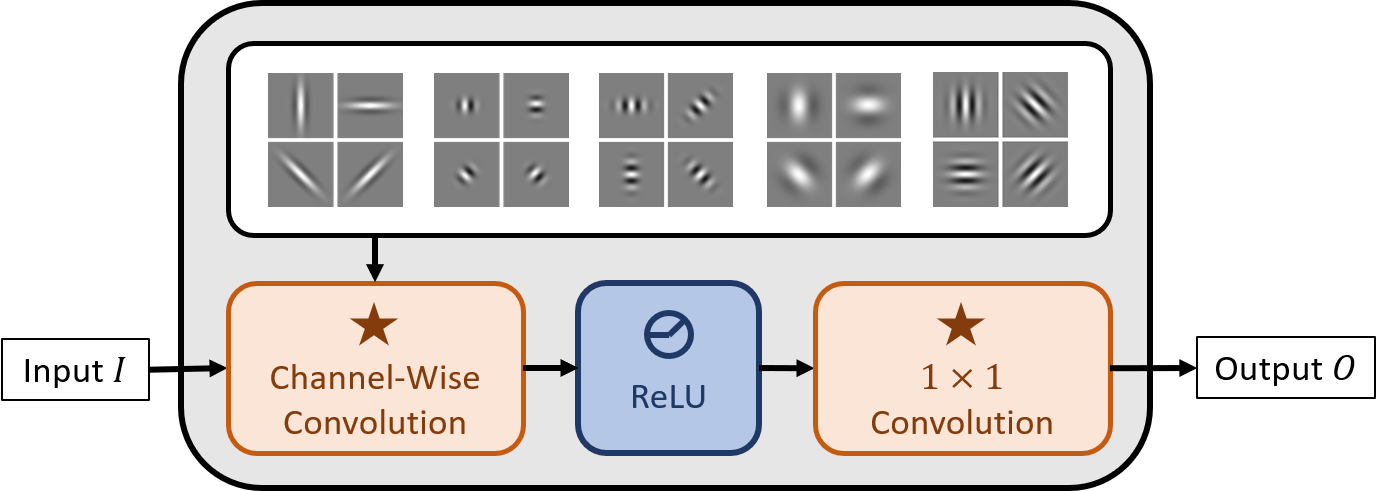}
        \caption{}
        \label{subfig:a}
    \end{subfigure}
    \begin{subfigure}{.53\textwidth}
        \includegraphics[width=\textwidth]{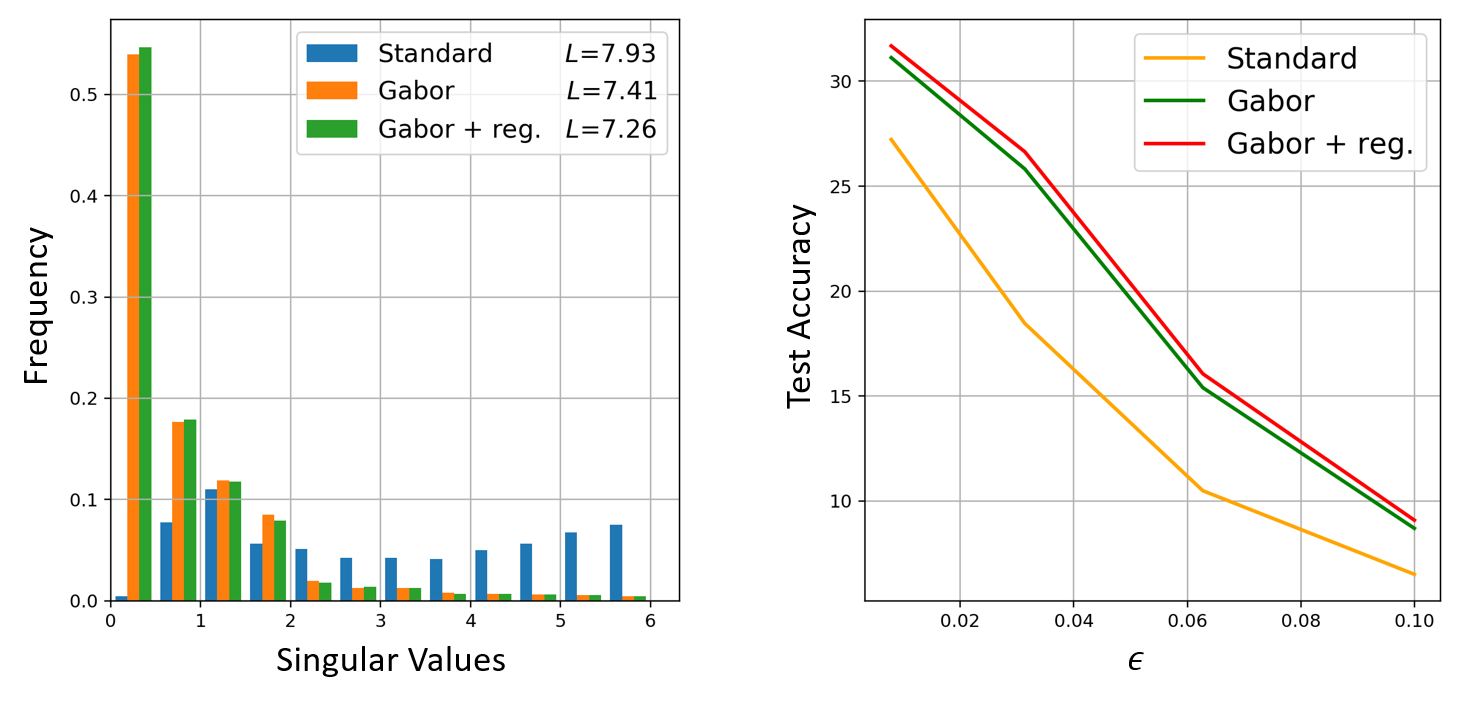}
        \caption{}
        \label{subfig:b}
    \end{subfigure}
    \caption{\textbf{Gabor layers and their effect on network robustness.} (a): Gabor layers convolve each channel of the input with a set of learned Gabor filters. As low-level filters, Gabor filters offer a natural approach to represent local signals. (b): Replacing standard convolutional layers with Gabor layers imposes structure in the distribution of singular values of the filters, reduces the Lipschitz constant of the filters (shown as $L$ in the legend of the left plot), and improves accuracy under adversarial attacks (right figure). The results shown here are for VGG16 on CIFAR100.}
    \label{fig:pull_image}
    \vspace{-0.65cm}
\end{figure}

The brittleness of DNNs in the presence of adversarial attacks has spurred much interest in the machine learning community, as evidenced by the emerging corpus of recent methods that focus on designing adversarial attacks \cite{goodfellow2014explaining,carlini_cw,Moosavi_Dezfooli_2016,structured_attack}. This phenomenon is far-reaching and widespread, and is of particular importance in real-world scenarios, \eg, autonomous cars \cite{Cao_2019,Chernikova_2019} and devices for the visually impaired \cite{visimpaired_DL}. The risks that this degenerate behavior poses underscore the need for models that are not only accurate, but also robust to adversarial attacks. 

Despite the complications that adversarial examples raise in modern computer vision, such inconveniences were not a major concern in the pre-DNN era. Many classical computer vision methods drew inspiration from the animal visual system, and so were designed to extract and use features that were meaningful to humans \cite{sift1_lowe,sift2_lowe,colorhistogram,Leung01ijcv}. As such, these methods were structured, generally comprehensible and, hence, better understood when compared to DNNs. In many cases, these methods even exhibited rigorous stability properties under robustness analysis \cite{reconstruction}. However, mainly due to large performance gaps on several tasks, classical methods were overshadowed by DNNs. It is precisely in the frontier between classical computer vision and DNNs that a stream of works arose to combine tools and insights from both worlds to improve performance. For instance, the works of \cite{admmnet,IstaNet} showed that introducing structured layers inspired by the classical compressed sensing literature can outperform pure learning-based DNNs. Moreover, Bai \etal \cite{bai2017deep_watershed} achieved large gains in performance in instance segmentation by introducing intuitions from the classical watershed transform into DNNs.

In this paper, searching for robustness in computer vision, we draw inspiration from biological vision, as the survival of species strongly depends on both the accuracy and robustness of the animal visual system. We note that Marr's and Julesz' work \cite{Marr:1982:VCI:1095712,Julesz1981} argues that the visual cortex initially processes low-level agnostic information, in which the input of the system is segmented according to blobs, edges, bars, curves and boundaries. Furthermore, Hubel and Wiesel \cite{hubel&wiesel} demonstrated that individual cells on the primary visual cortex respond to wave textures with different angles in an animal model, providing evidence that supports Marr's theory. Since Gabor filters \cite{gabor1946theory} are based on mathematical functions that are capable of modeling elements that resemble those that the animal visual cortices respond to, these filters became of customary use in the computer vision community, and have been used for texture characterization \cite{Julesz1981,Leung01ijcv}, character recognition \cite{stroke_novel}, edge detection \cite{202090}, and face recognition \cite{377961,1215407}. While several works examine their integration into DNNs \cite{Gaborfilter+fast-training,GaborCNN,gabornet}, none investigate the effect of introducing parameterized Gabor filters into DNNs on the robustness of these networks. Our work fills this gap in the literature, as we provide experimental results demonstrating the significant impact that such an architectural change has on improving robustness. Figure \ref{fig:pull_image} shows an overview of our work and results.

\textbf{Contributions}: Our main contributions are two-fold: \textbf{(1)} We propose a \textit{parameterized} Gabor-structured convolutional layer as a replacement for early convolutional layers in DNNs. We observe that such layers can have a remarkable impact on robustness. Thereafter, we analyze and derive an analytical expression for a Lipschitz constant of the Gabor filters, and propose a new training regularizer to further boost robustness. \textbf{(2)} We empirically validate our claims with a large number of experiments on different architectures (LeNet \cite{lecun_lenet}, AlexNet \cite{alexnet}, VGG16 \cite{vgg_simonyan} and Wide-ResNet \cite{wideresnet_zagoruyko}) and over several datasets (MNIST \cite{lecun_mnist}, SVHN \cite{SVHN}, CIFAR10 and CIFAR100 \cite{Cifars}). We show that introducing our proposed Gabor layers in DNNs induces a consistent boost in robustness at negligible cost, while preserving high generalizing test performance. In addition, we experimentally show that our novel regularizer based on the Lipschitz constant we derive can further improve adversarial robustness. For instance, we improve adversarial robustness on certain networks by almost \textbf{$18\%$} with $\ell_\infty$ bounded noise of $\nicefrac{8}{255}$. Lastly, we show empirically that combining this architectural change with adversarial training \cite{madry_defense,freeadv} can further improve robustness. 

%% file: sections/2related_work.tex
\section{Related Work}
 
\textbf{Integrating Gabor Filters with DNNs.} 
Several works attempted to combine Gabor filters and DNNs. For instance, the work of \cite{Gaborfilter+fast-training} showed that replacing the first convolutional layers in DNNs with Gabor filters speeds up the training procedure, while \cite{GaborCNN} demonstrated that introducing Gabor layers reduces the parameter count without hurting generalization accuracy. Regarding large scale datasets, Alekseev and Bobe \cite{gabornet} showed that the standard classification accuracy of AlexNet \cite{alexnet} on ImageNet \cite{russakovsky2015imagenet} can be attained even when the first convolutional filters are replaced with Gabor filters. Moreover, other works have integrated Gabor filters with DNNs for various applications, \eg, pedestrian detection \cite{gabor_pedestrian}, object recognition \cite{gabor_feat_nat_scene}, hyper-spectral image classification \cite{hyperspect_gaborcnn}, and Chinese optical character recognition \cite{gabor_chinese_char}. Likewise, in this work, we study the effects of introducing Gabor filters into various DNNs by means of a \textit{Gabor layer}, a convolution-based layer we propose in which the convolutional filters are constructed by a parameterized Gabor function with learnable parameters. Furthermore, and based on the well-defined spatial structure of these filters, we study the effect of these layers on robustness, and find encouraging results.

\textbf{Robust Neural Networks.} Recent work demonstrated that DNNs are vulnerable to adversarial perturbations. This susceptibility incited a stream of research that aimed to develop not only accurate but also robust DNNs. A straightforward approach to this nuisance is the direct augmentation of data corrupted with adversarial examples in the training set \cite{goodfellow2014explaining}. However, the performance of this approach can be computationally limited, since the amount of augmentation needed for a high dimensional input space is computationally prohibitive. Moreover, Papernot \etal \cite{defensive} showed that distilling DNNs into smaller networks can improve robustness. Another approach to robustness is through the functional perspective lens. For instance, Parseval Networks \cite{parseval_nets} showed that robustness can be achieved by regularizing the Lipschitz constant of each layer in a DNN to be smaller than $1$. In this work, along the lines of Parseval Networks \cite{parseval_nets}, and since Gabor filters can be generated by sampling from a continuous Gabor function, we derive an analytical closed form expression for the Lipschitz constant of the filters of the proposed Gabor layer. This derivation allows us to propose well-motivated regularizers that can encourage Lipschitz constant minimization, and then harness such regularizers to improve the robustness of networks with Gabor layers.

\textbf{Adversarial Training.} An orthogonal direction for obtaining robust models is through optimization of a saddle point problem, in which an adversary, whose aim is to maximize the objective, is introduced into the traditional optimization objective. In other words, instead of the typical training scheme, one can minimize the worst adversarial loss over all bounded energy (often measured in $\ell_\infty$ norm) perturbations around every given input in the training data. This approach is one of the most celebrated for training robust networks, and is now popularly known as adversarial training \cite{madry_defense}. However, this training comes at an inconvenient computational cost. To this regard, several works \cite{wong2020fastadv,freeadv,NIPS2019_YOPO} proposed faster and computationally-cheaper versions of adversarial training capable of achieving similar robustness levels. In this work, we use \say{free} adversarial training \cite{freeadv} in our experiments to further study Gabor layers and adversarial training as orthogonal approaches to achieve robustness. Our results show how adversarial training and Gabor layers interact positively and can, hence, be jointly used for enhancing network robustness.

%% file: sections/3methodology.tex
\section{Methodology}
As demonstrated by Hubel and Wiesel \cite{hubel&wiesel}, the first layers of visual processing in the animal brain are responsible for detecting low-level visual information. Since Gabor filters have the capacity to capture low-level representations, and inspired by the robust properties of the animal visual system, we hypothesize that Gabor filters possess inherent robustness properties that are transferable to other systems, perhaps even DNNs. In this section, we discuss our proposed Gabor layer and its implementation. Then, we derive a Lipschitz constant to the Gabor filter, and design a regularizer with aims at controlling the robustness properties of the layer by controlling the Lipschitz constant.

\subsection{Convolutional Gabor Filter as a Layer}

We start by introducing the Gabor functions, defined as follows:

\begin{equation}\label{eq:gabor_def}
\begin{aligned}
&G_{\theta}(x', y'; \sigma, \gamma, \lambda, \psi) :=  e^{-\sigma^2 \: (x'^2 + \gamma^2 y'^2)} \: \cos(\lambda x' + \psi)    \\
& \quad \quad x' = x\cos\theta - y\sin\theta \quad y' = x\sin\theta + y\cos\theta.
\end{aligned}
\end{equation}

To construct a discrete Gabor filter, we discretize $x$ and $y$ in (\ref{eq:gabor_def}) uniformly on a grid, where the number of grid samples determines the filter size. Given a fixed set of parameters $\{\sigma,\gamma,\lambda,\psi\}$, a grid $\{(x_i,y_i)\}_{i=1}^{k^2}$ of size $k \times k$, a rotation angle $\theta_j$ and a filter scale $\alpha_j$, computing Equation~(\ref{eq:gabor_def}) with a scale over the grid yields a single surface $ \alpha_j G_{\theta_j}(x',y';\sigma,\gamma,\lambda,\psi) \in \mathbb{R}^{1 \times k \times k}$, that we interpret as a filter for a convolutional layer. The learnable parameters~\cite{backprop} for such a function are given by the set $\mathcal{P} = \{\alpha_j,\sigma,\gamma,\lambda,\psi;\forall j=1,\dots,r\}$, where the rotations $\theta_j$ are restricted to be $r$ angles uniformly sampled from the interval $[0, \, 2\pi]$. Evaluating these functions results in $r$ rotated filters, each with a scale $\alpha_j$ defined by the set $\mathcal{F}_{\mathcal{P}} = \{\alpha_j G_{\theta_j}\}_{j=1}^r$. In this work, we consider several sets of learnable parameters $\mathcal{P}$, say $p$ of them, thus, the set of all Gabor filters (totaling to $rp$ filters) is given by the set $\mathcal{K} = \{\mathcal{F}_{\mathcal{P}_i}\}_{i=1}^p$. Refer to Figure~\ref{fig:gabor_operation} for a graphical guide on the construction of $\mathcal{K}$.

\begin{figure*}[t!]
    \centering
    \includegraphics[width=\textwidth]{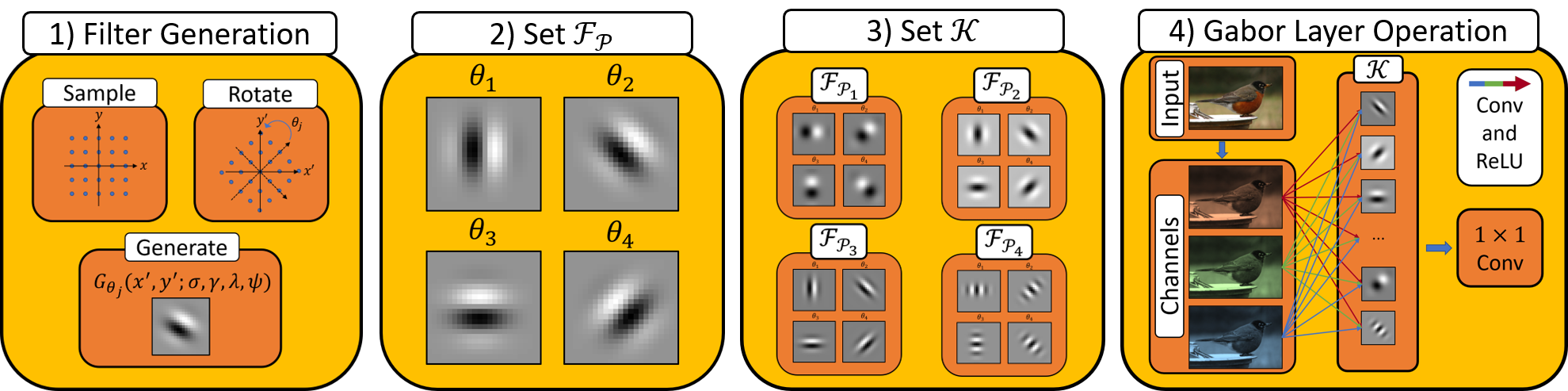}
    \caption{\textbf{Gabor layer operations.} \textit{(1)} We generate filters by rotating a sampled grid over multiple orientations and then evaluating the Gabor function according to set of parameters $\mathcal{P}$, to yield the set of filters $\mathcal{F}_\mathcal{P}$ \textit{(2)}. Then, we construct the total set of filters $\mathcal{K}$ \textit{(3)} by joining multiple sets $\mathcal{F}_{\mathcal{P}_i}$. Finally, the Gabor Layer operation \textit{(4)} separately convolves every filter in $\mathcal{K}$ with every channel from the input, applies ReLU non-linearity, and then applies a $1\times1$ convolution to the output features to get the desired number of output channels.}
    \label{fig:gabor_operation}
\end{figure*}

\subsection{Implementation of the Gabor Layer}

Given an input tensor $I$ with $m$ channels, $I \in \mathbb{R}^{m \times h \times w}$, the Gabor layers follow a depth-wise separable convolution-based \cite{Chollet_2017} approach to convolve the Gabor filters in $\mathcal{K}$ with $I$.
The tensor $I$ is first separated into $m$ individual channels. Next, each channel is convolved with each filter in $\mathcal{K}$, and then a ReLU activation is applied to the output. Formally, the Gabor layer with filters in the set $\mathcal{K}$ operating on some input tensor $I$ is presented as $\mathcal{R} = \{\text{ReLU}(I_i \star f_j), \: I_i \in \mathcal{I}, \: f_j \in \mathcal{K},~\forall i,j\}$ where $I_i = I(i,:,:) \in \mathbb{R}^{1 \times h \times w}$ and  $\star$ denotes the convolution operation. This operation produces $|\mathcal{R}| = mrp$ responses. Finally, the responses are stacked and convolved with a $1 \times 1$ filter with $n$ filters. Thus, the final response is of size $n\times h^\prime \times w^\prime$. Figure~\ref{fig:gabor_operation} shows an overview of this operation.

\subsection{Regularization}
A function $f : \mathbb{R}^n \rightarrow \mathbb{R}$ is $L$-Lipschitz if $\|f(x) - f(y)\| \leq L \|x - y\| $, $ \forall x,y \in \mathbb{R}^n$, where $L$ is the Lipschitz constant. Studying the Lipschitz constant of a DNN is essential for exploring the its robustness properties, since DNNs with a small Lipschitz constant enjoy a better-behaving backpropagated signal of gradients, improved computational stability~\cite{sedghi2018the}, and an enhanced robustness to adversarial attacks~\cite{parseval_nets}. 

Therefore, to train networks that are robust, Cisse \etal \cite{parseval_nets} proposed a regularizer that encourages the weights of the networks to be tight frames, which are extensions of orthogonal matrices to non-square matrices. The training procedure is, however, nontrivial to implement. Following the intuition of \cite{parseval_nets}, we study the continuity properties of the Gabor layer and derive an expression for a Lipschitz constant as a function of the parameters of the filters, $\mathcal{P}$. This expression allows for direct network regularization on the parameters of the Gabor layer corresponding to the fastest decrease in the Lipschitz constant of the layer. To this end, we present our main theoretical result, which allows us to apply regularization on the Lipschitz constant of the filters of a Gabor layer. 

\begin{theorem}\label{theorem:gabor_lipsch}
Given a Gabor filter $G_{\theta}(m, n; \sigma, \gamma, \lambda, \psi)$, a Lipschitz constant $L$ of the convolutional layer that has $G_\theta$ as its filter, with circular boundary conditions for the convolution, is given by:

\begin{equation*}
    L=\left(1 + \abs{X^\prime}e^{-\sigma^2 m_{*}^2}\right)\left(1 + \abs{Y^\prime}e^{-\sigma^2 \gamma^2n_{*}^2}\right),
\end{equation*}

\noindent where $X^\prime = X \setminus \{0\}$, $Y^\prime = Y \setminus \{0\}$, $X = \{x_i\}_{i=1}^{k^2}$ and $Y = \{y_i\}_{i=1}^{k^2}$ are sets of sampled values of the rotated $(x',y')$ grid where $\{0\} \in X,Y$, $m_{*} = \argmin_{x \in X^\prime} \abs{x}$ and $n_{*} = \argmin_{y \in Y^\prime} \abs{y}$.
\end{theorem}

\begin{proof}

To compute the Lipschitz constant of a convolutional layer, one must compute the largest singular value of the underlying convolutional matrix of the filter. For separable convolutions, this computation is equivalent to the maximum magnitude of the 2D Discrete Fourier Transform (DFT) of the Gabor filter $G_{\theta}$~\cite{sedghi2018the,bibi2018deep}. Thus, the Lipschitz constant of the convolutional layer is given by $L = \max_{u,v} \abs{\text{DFT} \left(G_{\theta}(m,n;\sigma,\gamma,\lambda,\psi)\right)}$, where $\text{DFT}$ is the 2D DFT over the coordinates $m$ and $n$ in the spatial domain, $u$ and $v$ are the coordinates in the frequency domain, and $|\cdot|$ is the magnitude operator. Note that $G_\theta$ can be expressed as a product of two functions that are independent of the sampling sets $X$ and $Y$ as follows:
\begin{equation}
    \begin{aligned}
    G_{\theta}(m,n; \sigma, \gamma, \lambda, \psi) :=  \underbrace{e^{-\sigma^2 m^2}\cos(\lambda m + \psi)}_{f(m;\sigma,\lambda,\psi)} \underbrace{e^{-\sigma^2 \gamma^2 n^2}}_{g(n;\sigma,\gamma)}. \nonumber
    \end{aligned}
\end{equation}

Thus, we have
\begin{align*}
    L & = \max_{u,v} \big|\text{DFT} \left(G_{\theta}(m,n;\sigma,\gamma,\lambda,\psi)\right)\big| \\
      & = \max_{u,v} \big|\sum_{m \in X}e^{-\omega_m um}f(m;\sigma,\lambda,\psi)\sum_{n \in Y}e^{-\omega_n vn}g(n;\sigma,\gamma)\big| \\
      & \leq \max_{u,v} \sum_{m \in X} \big|f(m;\sigma,\lambda,\psi)\big|\sum_{n \in Y}\big|g(n;\sigma,\gamma)\big|.
\end{align*}

Note that $\omega_m = \frac{j2\pi}{\abs{X}}$, $\omega_n = \frac{j2\pi}{\abs{Y}}$ and $j^2 = -1$. The last inequality follows from Cauchy–Schwarz and the fact that $|e^{-\omega_m um}| = |e^{-\omega_n vn}| = 1$. Note that since $\abs{g(n;\sigma,\gamma)} = g(n;\sigma,\gamma)$, and $\abs{f(m;\sigma,\lambda,\psi)} \leq e^{-\sigma^2 m^2}$ we have that:
\begin{align*}
    L &\leq \sum_{m \in X}e^{-\sigma^2 m^2}\sum_{n \in Y}e^{-\sigma^2 \gamma^2 n^2}
    \leq  \left(1 + \abs{X'}e^{-\sigma^2m_{*}^2}\right)\left(1 +\abs{Y'}e^{-\sigma^2 \gamma^2n_{*}^2}\right).
\end{align*}

The last inequality follows by construction, since we have $\{0\} \in X,Y$, $i.e.$, the choice of uniform grid contains the $0$ element in both $X$ and $Y$, regardless of the orientation $\theta$, where we define $m_{*} = \argmin_{x \in X^\prime} \abs{x}$, and $n_{*} = \argmin_{y \in Y^\prime} \abs{y}$. 
\begin{flushright}
$\square$
\end{flushright}
\end{proof}

\subsection{Lipschitz Constant Regularization}

Theorem~\ref{theorem:gabor_lipsch} provides an explicit expression for a Lipschitz constant of the Gabor filter as a function of its parameters. Note that the expression we derived decreases exponentially fast with $\sigma$. In particular, we note that, as $\sigma$ increases, $G_{\theta}$ converges to a scaled Dirac-like surface. Hence, this Lipschitz constant is minimized when the filter resembles a Dirac-delta. Therefore, to train DNNs with improved robustness, one can minimize the Lipschitz constant we find in Theorem \ref{theorem:gabor_lipsch}. Note that the Lipschitz constant of the network can be upper bounded by the product of the Lipschitz constants of individual layers. Thus, decreasing the Lipschitz constant we provide in Theorem \ref{theorem:gabor_lipsch} can aid in decreasing the overall Lipschitz constant of a DNN, and thereafter enhance the network's robustness. To this end, we propose the following regularized loss:
\begin{equation}
\label{eq:modified_loss1}
    \mathcal{L} = \mathcal{L}_{\text{ce}} - \beta \: \sum_i \sigma_i^2,
\end{equation}
where $\mathcal{L}_{\text{ce}}$ is the typical cross-entropy loss and $\beta > 0$ is a trade-off parameter. The loss in Equation (\ref{eq:modified_loss1}) can lead to
unbounded solutions for $\sigma_i$. To alleviate this behavior, we also propose the following loss:
\begin{equation}
\label{eq:modified_loss2}
    \mathcal{L} = \mathcal{L}_{\text{ce}} - \beta \sum_i \left(\mu\: \tanh{\sigma_i}\right)^2,
\end{equation}
where $\mu$ is a scaling constant for $\tanh{\sigma}$. 
In the following section, we present experiments showing the effect of our Gabor layers on network robustness. Specifically, we show the gains obtained from the architectural modification of introducing Gabor layers, the introduction of our proposed regularizer, and the addition of adversarial training to the overall pipeline.

%% file: sections/4experiments.tex
\section{Experiments}
To demonstrate the benefits and impact on robustness of integrating our Gabor layer to DNNs, we conduct extensive experiments with LeNet \cite{lecun_lenet}, AlexNet \cite{alexnet}, VGG16 \cite{vgg_simonyan}, and Wide-ResNet \cite{wideresnet_zagoruyko} on the MNIST \cite{lecun_mnist}, CIFAR10, CIFAR100 \cite{Cifars} and SVHN \cite{SVHN} datasets. In each of the aforementioned networks, we replace up to the first three convolutional layers with Gabor layers, and measure the impact of the Gabor layers in terms of accuracy, robustness, and  the distribution of singular values of the layers. Moreover, we perform experiments demonstrating that the robustness of the Gabor-layered networks can be enhanced furthermore by using the regularizer we propose in Equations~(\ref{eq:modified_loss1}) and (\ref{eq:modified_loss2}), and even when jointly employing regularization and adversarial training \cite{freeadv}.

\subsection{Implementation Details}\label{subsec:impl_details}
We train all networks with stochastic gradient descent with weight decay of $5\times10^{-4}$, momentum of $0.9$, and batch size of $128$. For MNIST, we train the networks for $90$ epochs with a starting learning rate of $10^{-2}$, which is multiplied by a factor of $10^{-1}$ at epochs $30$ and $60$. For SVHN, we train models for $160$ epochs with a starting learning rate of $10^{-2}$ that is multiplied by a factor of $10^{-1}$ at epochs $80$ and $120$. For CIFAR10 and CIFAR100, we train the networks for $300$ epochs with a starting learning rate of $10^{-2}$ that is multiplied by a factor of $10^{-1}$ every $100$ epochs.

\subsection{Robustness Assessment} 

Following common practice in the literature for empirically evaluating robustness properties \cite{madry_defense,freeadv}, we assess the robustness of a DNN by measuring its prediction accuracy when the input is probed with adversarial attacks, which is widely referred to in the literature as \say{adversarial accuracy}. We also measure the \say{flip rate}, which is defined as the percentage of instances of the test set for which the predictions of the network changed when under adversarial attacks.

Formally, if $x \in \mathbb{R}^d$ is some input to a classifier $C:\mathbb{R}^d \rightarrow \mathbb{R}^k$, $C(x)$ is the prediction of $C$ at input $x$. Then, $x^{\text{adv}} = x + \eta$ is an adversarial example if the prediction of the classifier has changed, \ie $C(x^{\text{adv}}) \neq C(x)$. Both $\eta$ and $x^{\text{adv}}$ must adhere to constraints, namely: \textit{(i)} the $\ell_p$-norm of $\eta$ must be bounded by some $\epsilon$, $i.e.$, $\|\eta\|_p \leq \epsilon$, and \textit{(ii)} $x^{\text{adv}}$ must lie in the space of valid instances $X$, \textit{i.e.}, $x^{\text{adv}}\in [0, 1]^d$. A standard approach to constructing $x^{\text{adv}}$ for some input $x$ is by running Projected Gradient Descent (PGD)~\cite{madry_defense} with $x$ as an initialization for several iterations. For some loss function $\mathcal{L}$, a PGD iteration projects a step of the Fast Gradient Sign Method~\cite{goodfellow2014explaining} onto the valid set $\mathcal{S}$ which is defined by the constraints on $\eta$ and $x^{\text{adv}}$. Formally, one iteration of PGD attack is:
$$ 
x^{k+1} = \prod_{\mathcal{S}}\left(x^k + \delta \: \text{sign}\left(\nabla_{x^{k}}\mathcal{L}(x^k,y)\right)\right),
$$
where $\prod_{\mathcal{S}}$ is the projection operator onto $\mathcal{S}$ and $y$ is the label. In our experiments, we consider attacks where $\eta$ is $\epsilon$-$\ell_{\infty}$ bounded. For each image, we run PGD for $200$ iterations and perform $10$ random restarts inside the $\epsilon$-$\ell_{\infty}$ ball centered in the image. Following prior art \cite{wong2020fastadv}, we set $\epsilon \in \{0.1, 0.2, 0.3\}$ for MNIST and $\epsilon \in \{\nicefrac{2}{255}, \nicefrac{8}{255}, \nicefrac{16}{255}\}$ for all other datasets. Throughout our experiments, we assess robustness by measuring the adversarial accuracies and flip rates when under PGD attacks.

\subsection{Performance of Gabor-Layered Architectures}
\begin{table}[t]
\centering
\caption{\label{tab:recover_accuracy} \textbf{Test set accuracies on different datasets of various baselines, and their Gabor-layered versions.} Gabor-layered architectures can recover the accuracies of their standard counterparts while providing robustness. $\Delta$ is the absolute difference between the baselines and the Gabor-layered architectures.}
\vspace{0.3cm}
\centering
\begin{tabular}{c|c|ccc}
\toprule
\text{    } Dataset \text{    } & \text{  } Architecture \text{  }   & \text{  } Baseline \text{  }  & Gabor & \text{  } $\Delta$ \text{  } \\
\midrule
MNIST    & LeNet        & 99.36     & 99.03 & 0.33 \\\hline 
SVHN     & WideResNet   & 96.62     & 96.70 & 0.08 \\\hline
SVHN     & VGG16        & 96.52     & 96.18 & 0.34 \\\hline
CIFAR10  & VGG16        & 92.03     & 91.35 & 0.68 \\\hline
CIFAR100 & AlexNet      & 46.48     & 45.15 & 1.33 \\\hline
CIFAR100 & WideResNet   & 77.68     & 76.86 & 0.82 \\\hline
CIFAR100 & VGG16        & 67.54     & 64.49 & 3.05 \\\hline
\toprule
\end{tabular}
\end{table}

The Gabor function in Equation~(\ref{eq:gabor_def}) restricts the space of patterns attainable by the Gabor filters. However, this set of patterns is aligned with what is observed in practice in the early layers of many standard architectures \cite{atanov2018deep,alexnet}. This observation follows the intuition that DNNs learn hierarchical representations, with early layers detecting lines and blobs, and deeper layers learning semantic information \cite{lecun-hierarchical}. By experimenting with Gabor layers on various DNNs, we find that Gabor-layered DNNs recover close-to-standard, and sometimes better, test-set accuracies on several datasets. In Table~\ref{tab:recover_accuracy}, we report the test-set accuracies of several dataset-network pairs for standard DNNs and their Gabor-layered counterparts. We show the absolute difference in performance in the last column. 

Moreover, in Figure \ref{fig:behaviour_of_gabor}, we provide a visual comparison between the patterns learned by AlexNet in its original implementation \cite{alexnet} and those learned in the Gabor-layered version of AlexNet (trained on CIFAR100). We observe that filters in the Gabor layer converge to filters that are similar to those found in the original implementation of AlexNet, where we observe blob-like structures and oriented edges and bars of various sizes. Note that both sets of filters are, in turn, similar to filter banks traditionally used in computer vision, as those proposed by Leung and Malik \cite{Leung01ijcv}. Next, we show that the Gabor-layered networks highlighted in Table~\ref{tab:recover_accuracy} not only achieve test set accuracies as high as those of standard DNNs, but also enjoy better robustness properties for free.

\begin{figure}[t]
    \centering
    \begin{subfigure}{.47\textwidth}
        \includegraphics[width=\textwidth]{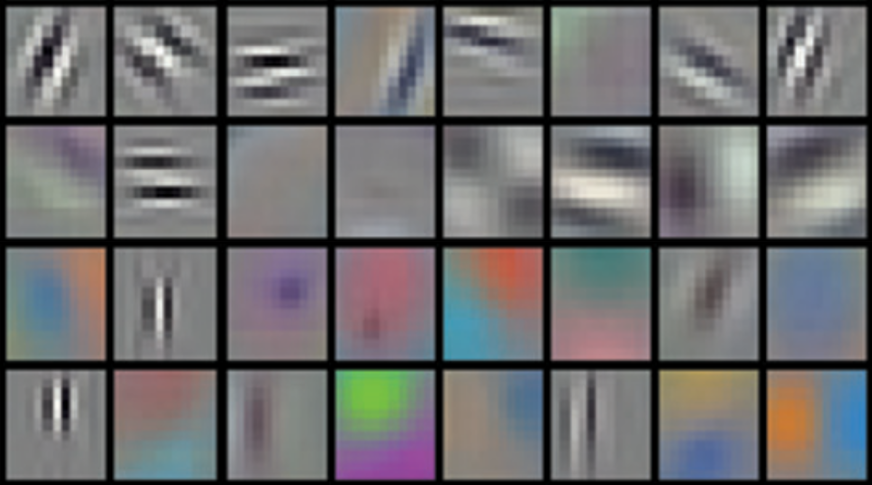}
        \caption{}
        \label{subfig:alexnet_filters}
    \end{subfigure}
    \begin{subfigure}{.47\textwidth}
        \includegraphics[width=\textwidth]{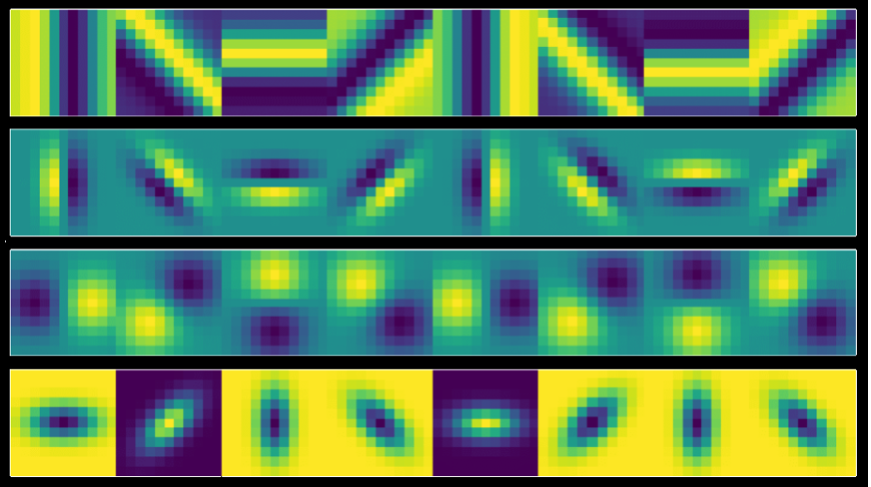}
        \caption{}
        \label{subfig:gabor_filters}
    \end{subfigure}
    \caption{\textbf{Comparison between filters learned by AlexNet and by Gabor-layered AlexNet.} (a) Various filters learned in the first convolutional layer of AlexNet in its original implementation \cite{alexnet}. (b) Several filters learned in the Gabor-layered version of AlexNet. Each column in (b) is a different orientation of the same filter, while each row represents a different set of parameters of the Gabor function. Note that standard convolutional layers use multiple-channeled filters, while Gabor layers use one-dimensional filters. Compellingly, both sets of filters present blobs and oriented edges and bars of various sizes and intensities. Best viewed in color.}
    \label{fig:behaviour_of_gabor}
\end{figure}

\subsection{Distribution of Singular Values}\label{subsec:sing_vals}
The Lipschitz constant of a network is an important quantity in the study of the network's robustness properties, since it is a measure of the variation of the network's predictions under input perturbations. Hence, in this work we study the distribution of singular values and, as a consequence, the Lipschitz constant of the filters of the layers in which we introduced Gabor layers instead of regular convolutional layers, in a similar fashion to \cite{parseval_nets}. In Figure~\ref{fig:singular_values_wb} we report box-plots of the distributions of singular values for the first layer of LeNet trained on MNIST, and the first three layers of VGG16 trained on CIFAR100. Each plot shows the distribution of singular values of the standard architectures (S), Gabor-layered architectures (G), and Gabor-layered architectures trained \textit{with} the regularizers (G+r) proposed in Equations (\ref{eq:modified_loss1}) and (\ref{eq:modified_loss2}).

Figure \ref{fig:singular_values_wb} demonstrates that the singular values of the filters in Gabor layers tend to be concentrated around smaller values, while also being distributed in smaller ranges than those of their standard convolutional counterparts, as shown by the interquartile ranges. Additionally, in most cases, the Lipschitz constant of the filters of Gabor layers, \ie the top notch of each box-plot, is smaller than that of standard convolutional layers.

Moreover, we find that training Gabor-layered networks with the regularizer we introduced in Equations~(\ref{eq:modified_loss1}) and (\ref{eq:modified_loss2}) incites further reduction in the singular values of the Gabor filters, as shown in Figure~\ref{fig:singular_values_wb}. For instance, the Gabor-layered version of LeNet trained on MNIST has a smaller interquartile range of the singular values, but still suffers from a large Lipschitz constant. However, by jointly introducing both the Gabor layer \textit{and} the regularizer, the Lipschitz constant decreases by a factor of almost 5. This reduction in the Lipschitz constant is consistent in all layers of VGG16 trained on CIFAR100.

In the following section, we show the bulk of our experiments, in which we investigate the effect on robustness of introducing Gabor layers into DNNs, by conducting a study across different architectures and datasets.

\begin{figure*}[t]
\centering
\includegraphics[width=0.7\textwidth]{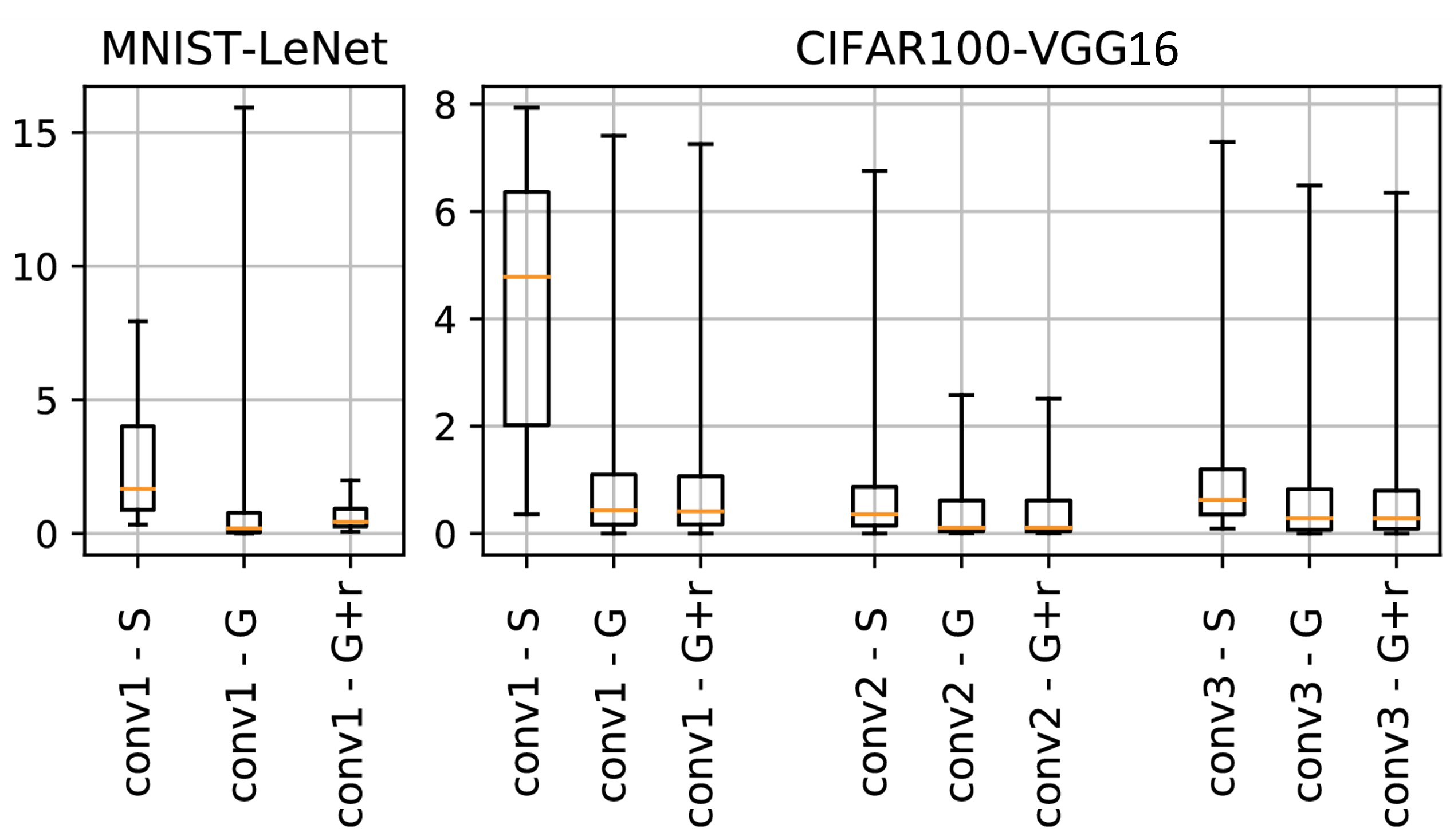}
\caption{\textbf{Box-plot representation of the distribution of singular values in layers of LeNet and VGG16.} Left: LeNet on MNIST. Right: VGG16 on CIFAR100. S: Standard; G: Gabor-layered; G+r: Gabor-layered with regularization. The top notch of each box-plot corresponds to the maximum value of the distribution, \ie the Lipschitz constant of the layer. 
}
\label{fig:singular_values_wb}
\end{figure*}

\subsection{Robustness in Gabor-Layered Architectures}
After observing significant differences in the distribution of singular values between standard convolutional layers and Gabor layers, we now study the impact on robustness that Gabor layers introduce. We study the robustness properties of different architectures trained on various datasets when Gabor layers are introduced in the first layers of each network. The modifications that we perform on each architecture are: 

\begin{itemize}
    \item \textbf{LeNet.} We replace the first layer with a Gabor layer with $p = 2$. 
    
    \item \textbf{AlexNet.} We replace the first layer with a Gabor layer with $p = 7$.
    
    \item \textbf{WideResNet.} We replace the first layer with a Gabor layer with $p = 4$ for SVHN, and $p=3$ for CIFAR100.
    
    \item \textbf{VGG16.} We replace the first three layers with Gabor layers with parameters $p = 3$, $p = 1$ and $p = 3$, respectively. 
\end{itemize}

We present the details for the choice of $p$ in the \textbf{supplementary material}.

\textbf{Standard Architectures \textit{vs.} Gabor-Layered Architectures.} We report the adversarial accuracies on both standard architectures and Gabor-layered architectures, where we refer to each with \say{S} and \say{G}, respectively. Table \ref{tab:maintable} presents results on SVHN, CIFAR10 and CIFAR100, and Table \ref{tab:mnist} presents results on MNIST. We observe that Gabor-layered architectures consistently outperform their standard counterparts across datasets, and can provide up to a $9\%$ boost in adversarial robustness. For instance, with $\epsilon = \nicefrac{8}{255}$ attacks, introducing Gabor layers into VGG16 boosts adversarial accuracy from $23.63$ to $30.11$ ($6.48\%$ relative increment) and from $5.84$ to $14.57$ ($8.73\%$ relative increment) on CIFAR10 and SVHN (Table \ref{tab:maintable}), respectively. For LeNet on MNIST, under an  $\epsilon = 0.2$ attack, introducing Gabor layers can boost adversarial accuracy from $4.39\%$ to $7.94\%$ ($80\%$ relative increment). Additionally, we report the flip rates for these experiments in the \textbf{supplementary material}. On the flip rates, we conduct a similar analysis, which yields equivalent conclusions: introducing Gabor layers leads to boosts in robustness. 

Moreover, to explore whether the robustness effect of using Gabor-layers holds in the regime of large-scale datasets, we conduct an experiment on ImageNet \cite{russakovsky2015imagenet}, which we report in the \textbf{supplementary material}. Due to limitations in computational resources, we conduct attacks only for $\epsilon \in \{\nicefrac{8}{255},\nicefrac{16}{255}\}$. In this experiment, we again measure adversarial accuracy and flip rates, and observe similar gains in robustness in terms of flip rates.

It is worthwhile to note that the increase in robustness we observe in the Gabor-layered networks came \textit{solely} from an architectural change, \ie replacing convolutional layers with Gabor layers, without there being any other modification. Our experimental results demonstrate that: (\textbf{1}) Simply introducing Gabor filters, in the form of Gabor layers, as low-level feature extractors in DNNs provides beneficial effects to robustness. (\textbf{2}) Such robustness improvements are consistent across datasets and architectures. Inspired by these results, we now investigate the robustness effects of using our proposed regularization approach, as proposed in Equations (\ref{eq:modified_loss1}) and (\ref{eq:modified_loss2}).

\begin{table*}[t]
\centering
\caption{\label{tab:maintable}\textbf{Adversarial accuracy comparison}. We compare Standard (S), Gabor-layered (G), and \textit{regularized} Gabor-layered (G+r) architectures. For each attack strength ($\epsilon$), the highest performance is in \textbf{bold}; second-highest is \unde{underlined}.}
\vspace{0.2cm}
\resizebox{\columnwidth}{!}{
\begin{tabular}{c c||c c c|c c c|c c c}
\toprule
\multicolumn{2}{c|}{$\epsilon$}                     &  \multicolumn{3}{c|}{$\nicefrac{2}{255}$}     &  \multicolumn{3}{c|}{$\nicefrac{8}{255}$}     & \multicolumn{3}{c}{$\nicefrac{16}{255}$}      \\\midrule
Dataset&Network                & S             & G             & G+r           & S             & G             & G+r           & S             & G             & G+r           \\\midrule
SVHN&WRN            & 40.27         & \unde{49.35}  & \textbf{53.36}& 1.02          & \unde{1.03}   & \textbf{1.13} & \unde{1.32}   & \textbf{1.36} & 1.19          \\
SVHN&VGG16                    & 57.86         & \unde{62.88}  & \textbf{64.03}& 5.84          & \unde{14.57}  & \textbf{15.99}& 2.33          & \unde{7.98}   & \textbf{8.88} \\
CIFAR10&VGG16                  & 34.22         & \unde{37.60}  & \textbf{38.07}& 23.63         & \unde{30.11}  & \textbf{30.69}& 13.88         & \unde{19.50}  & \textbf{19.93}\\
CIFAR100&AN             & \textbf{15.05}& \unde{14.77}  & 14.68         & 4.80          & \unde{7.71}   & \textbf{7.88} & 5.37          & \unde{6.25}   & \textbf{6.67} \\
CIFAR100&WRN   & 4.52 & \unde{8.06} & \textbf{9.33} & 2.38 & \unde{2.92} & \textbf{3.07} & 1.65 & \unde{2.4} & \textbf{2.59} \\
CIFAR100&VGG16          & 27.22         & \unde{31.12}  & \textbf{31.68}& 18.46         & \unde{25.82}  & \textbf{26.64}& 10.49         & \unde{15.40}  & \textbf{16.06}
\\\toprule
\end{tabular}}
\end{table*}

\textbf{Robustness Effects of Introducing Regularization.} To better control robustness and to regularize the Lipschitz constant we derived in Theorem \ref{theorem:gabor_lipsch}, we proposed two regularizers for the loss function as per Equations (\ref{eq:modified_loss1}) and (\ref{eq:modified_loss2}). As we noted in Subsection \ref{subsec:sing_vals} (refer to Figure \ref{fig:singular_values_wb}), Gabor-layered architectures inherently tend to have lower singular values than their standard counterparts. Additionally, upon training the same architectures \textit{with} the proposed regularizer, we observe that Gabor layers tend to enjoy further reduction in the Lipschitz constant. These results suggest that such architectures may have enhanced robustness properties as a consequence.

\begin{table*}[t]
\centering
\small
\caption{\label{tab:mnist}\small\textbf{Adversarial accuracy comparison on MNIST}. We compare Standard (S), Gabor-layered (G), and \textit{regularized} Gabor-layered (G+r) architectures. For each attack strength ($\epsilon$), the highest performance is in \textbf{bold}; second-highest is \unde{underlined}.}
\vspace{0.2cm}
\begin{tabular}{c c||c c c|c c c|c c c}
\toprule
\multicolumn{2}{c|}{$\epsilon$}                     &  \multicolumn{3}{c|}{$0.1$}     &  \multicolumn{3}{c|}{$0.2$}     & \multicolumn{3}{c}{$0.3$}      \\\midrule
Dataset&Network                & S             & G             & G+r           & S             & G             & G+r           & S             & G             & G+r           \\\midrule
MNIST&LeNet                & 80.04         & \unde{80.58}  & \textbf{88.42}& 4.39          & \unde{7.94}   & \textbf{22.69}& 0.44          & \textbf{0.78} & \unde{0.76} \\\toprule
\end{tabular}
\end{table*}

To assess the role of the proposed regularizer on robustness, we train Gabor-layered architectures from scratch following the same parameters from Subsection~\ref{subsec:impl_details} and include the regularizer. We present the results in Tables \ref{tab:maintable} and \ref{tab:mnist}, where we refer to these regularized architectures as \say{G+r}. We observe that, in most cases, adding the regularizer improves adversarial accuracy. For instance, for LeNet on MNIST, the regularizer improves adversarial accuracy over the Gabor-layered architecture without any regularization by $8\%$ and $14\%$ with $\epsilon = \nicefrac{2}{255}, \nicefrac{8}{255}$ attacks, respectively. This improvement is still present in more challenging datasets. For instance, for VGG16 on SVHN under attacks with $\epsilon = \nicefrac{2}{255}$ and $\epsilon=\nicefrac{8}{255}$, we observe increments of over $1\%$ from the regularized architecture with respect to its non-regularized equivalent. For the rest of the architectures and datasets, we observe modest but, nonetheless, sustained increments in performance. We report the flip rates for these experiments in the \textbf{supplementary material}. The conclusions obtained from analyzing the flip rates are analogous to what we conclude from the adversarial accuracies: applying regularization on these layers provides minor but consistent improvements in robustness.

It is worthy to note that, although the implementation of the regularizer is trivial and that we perform optimization including the regularizer for the same number of epochs as regular training, our results still show that it is possible to achieve desirable robustness properties. We expect that substantial modifications and optimization heuristics can be applied in the training procedure, with aims at stronger exploitation of the insights we have provided here, and most likely resulting in more significant boosts in robustness.

\subsection{Effects of Adversarial Training}

Adversarial training \cite{madry_defense} has become the standard approach for tackling robustness. In these experiments, we investigate how Gabor layers interact with adversarial training. We study whether the increments in robustness we observed when introducing Gabor layers can still be observed in the regime of adversarially-trained models. To study such interaction, we adversarially train standard, Gabor-layered and regularized Gabor-layered architectures, and then compare their robustness properties. We use the adversarial training described in \cite{freeadv}, with $8$ mini-batch replays, and $\epsilon = \nicefrac{8}{255}$. 

In Table \ref{tab:adv_training2}, we report adversarial accuracies for AlexNet on CIFAR10 and CIFAR100 under $\epsilon = \nicefrac{8}{255}$ attacks. Even in the adversarially trained networks-regime, our experiments show that \textbf{(1)} Gabor-layered architectures outperform their standard counterparts, and \textbf{(2)} regularization of Gabor-layered architectures provides substantial improvements in robustness with respect to their non-regularized equivalents. Such results demonstrate that Gabor layers represent an orthogonal approach towards robustness and, hence, that Gabor layers and adversarial training can be jointly harnessed for enhancing the robustness of DNNs.

The results we present here are empirical evidence that using closed form expressions for filter-generating functions in convolutional layers can be exploited for the purpose of increasing robustness in DNNs. We refer the interested reader to the \textbf{supplementary material} for the rest of the experimental results.

\begin{table}[t]
\centering
\caption{\label{tab:adv_training2} Adversarial accuracy with $\epsilon = \nicefrac{8}{255}$.}
\vspace{0.1cm}
\centering
\begin{tabular}{c|c|c||c|c||c|c}
\multicolumn{3}{c}{}   & \multicolumn{2}{c}{CIFAR10} & \multicolumn{2}{c}{CIFAR100} \\
\toprule
Gabor   & Reg.  & Adv. training & AlexNet       & VGG16             & AlexNet       & VGG16             \\\midrule 
        &       & \cmark        & 19.24         & 41.95             & 13.48         & 18.49             \\\hline 
\cmark  &       & \cmark        & 20.26         & 42.41             &10.94          & \textbf{19.66}    \\\hline 
\cmark  &\cmark & \cmark        & \textbf{22.15}& \textbf{44.02}& \textbf{13.62}& 19.41             \\\toprule
\end{tabular}
\end{table}

%% file: sections/5conclusions.tex
\section{Conclusions}
In this work, we study the effects in robustness of architectural changes in convolutional neural networks. We show that introducing Gabor layers consistently improves the robustness across various neural network architectures and datasets. We also show that the Lipschitz constant of the filters in these Gabor layers tends to be lower than that of traditional filters, which was theoretically and empirically shown to be beneficial to robustness \cite{parseval_nets}. Furthermore, theoretical analysis allows us to find a closed form expression for a Lipschitz constant of the Gabor filters. We then leverage this expression as a regularizer in the pursuit of enhanced robustness, and validate its usefulness experimentally. Finally, we study the interaction between Gabor layers, our regularizer, and adversarial training, and show that the benefits of using Gabor layers are still observed when deep learning models are specifically trained for the purpose of adversarial robustness, showing that Gabor layers can be jointly used with adversarial training for further enhancements in robustness.

%% file: sections/6acknowledgments.tex
\textbf{Acknowledgments.} This work was partially supported by the King Abdullah University of Science and Technology (KAUST) Office of Sponsored Research.

%% file: sections/supplemental.tex
\begin{center}
    \Large{\textbf{Supplementary Material}}
\end{center}

Next, we present the Supplementary Material for the paper \say{Gabor Layers Enhance Network Robustness}. In Section \ref{sec:suppl_pspace} we report the details for the choice of the $p$ parameter of the Gabor layers for the various experiments we present in the paper; Section \ref{sec:suppl_fliprates} shows the flip rates for all the experiments we reported; Section \ref{sec:suppl_imagenet} presents the results for one experiment on ImageNet \cite{russakovsky2015imagenet}, both in terms of adversarial accuracy and flip rate; in Section \ref{sec:suppl_singvals} we present several comparisons of distributions of singular values of standard convolutional layers \textit{vs.} Gabor layers; finally, Section \ref{sec:suppl_filtervis} depicts visualizations of the filters learned in the Gabor layers when various architectures are trained on numerous datasets.

\section{Number of families $p$ search space}\label{sec:suppl_pspace}
We report the details for the search for the $p$ parameter in Tables \ref{tab:lenet-mnist-p} through \ref{tab:wrn-svhn-p2}.

\begin{table*}
\centering
\caption{\label{tab:lenet-mnist-p}\textbf{Search Space for LeNet on MNIST} Test set accuracies and flip rates on MNIST with LeNet. Only the first convolutional layer was enhanced with $p$ families.}
\small
\begin{tabular}{c||c c|c c|c c}
\toprule
            & \multicolumn{2}{c|}{Standard}     &  \multicolumn{2}{c|}{$p=2$}    & \multicolumn{2}{c}{$p=3$}\\\midrule
$\epsilon$  & Accuracy      & Flip Rate         & Accuracy           & Flip Rate     & Accuracy          & Flip Rate     \\ \midrule
$0$         & 99.22         & -                 &  99.03             & -             & 99.10             & -             \\ \hline
$0.1$       & 80.04         & 19.53             &  80.58             & 18.88         & 62.98             & 36.54         \\ \hline
$0.2$       & 4.39          & 95.47             &  7.94              & 91.85         & 2.22              & 97.67         \\ \hline
$0.3$       & 0.44          & 99.7              &  0.78              & 99.24         & 0.41              & 99.68         \\ \toprule
\end{tabular}
\end{table*}

\begin{table*}
\centering
\caption{\label{tab:an-cifar100-p}\textbf{Search Space for AlexNet on CIFAR100.} Test set accuracies comparison with the standard AlexNet and his enhanced versions. Only the first convolutional layer was enhanced. S stands for standard.}
\small
\begin{tabular}{c||c||c|c|c|c|c|c|c|c}
\multicolumn{10}{c}{\textbf{Accuracy}}\\\toprule
$\epsilon$          & S     & $p=2$ & $p=3$ & $p=4$ & $p=5$ & $p=6$ & $p=7$ & $p=8$ & $p=9$ \\\midrule
$0$                 & 46.48 & 42.14 & 42.81 & 41.99 & 43.19 & 42.32 & 45.15 & 42.69 & 43.36\\ \hline
$\nicefrac{2}{255}$ & 15.08 & 11.84 & 13.01 & 13.41 & 13.62 & 13.55 & 14.77 & 12.82 & 13.16\\ \hline
$\nicefrac{8}{255}$ & 4.80  & 6.82  & 7.40  & 6.74  & 7.81  & 7.01  & 7.71  & 7.86  & 7.73\\ \hline
$\nicefrac{16}{255}$& 5.37  & 5.92  & 6.12  & 5.75  & 6.23  & 5.34  & 6.25  & 6.07  & 6.29\\ \toprule
\end{tabular}
\end{table*}

\begin{table*}
\centering
\caption{\label{tab:vgg-cifar100-p1}\textbf{Search Space for VGG16 on CIFAR100 -- First Layer.} Test set accuracies comparison with the standard VGG16 and his enhanced versions. The first Gabor layer number of families $p_1$ is explored. S stands for standard.}
\small
\begin{tabular}{c||c||c|c|c|c|c|c}
\multicolumn{8}{c}{\textbf{Accuracy}}\\\toprule
$\epsilon$          & S     & $p_1=1$ & $p_1=2$ & \textbf{$p_1=3$} & $p_1=4$ & $p_1=5$ & $p_1=6$ \\\midrule
$0$                 & 67.54 & 65.82 & 67.35 & 67.05 & 68.65 & 68.10 & 66.23 \\ \hline
$\nicefrac{2}{255}$ & 27.22 & 28.35 & 27.89 & 27.31 & 27.83 & 27.32 & 24.83 \\ \hline
$\nicefrac{8}{255}$ & 18.46 & 23.05 & 20.78 & 20.05 & 20.66 & 19.44 & 18.95 \\ \hline
$\nicefrac{16}{255}$& 10.49 & 13.92 & 12.18 & 11.30 & 11.02 & 10.64 & 11.29\\ \toprule
\end{tabular}
\end{table*}

\begin{table*}
\centering
\caption{\label{tab:vgg-cifar100-p2}\textbf{Search Space for VGG16 on CIFAR100 -- Second Layer.} Test set accuracies comparison with the standard VGG16 and his enhanced versions. The first convolutional layer is modified with $p_1=3$ families. The second Gabor layer number of families $p_2$ is explored. S stands for standard.}
\small
\begin{tabular}{c||c||c||c|c|c|c|c}
\multicolumn{8}{c}{\textbf{Accuracy}}\\\toprule
$\epsilon$          & S     & $p_1=3$ & \textbf{$p_2=1$} & $p_2=2$ & $p_2=3$ & $p_2=4$ & $p_2=5$ \\ \midrule
$0$                 & 67.54 & 67.05   & 64.10   & 67.68   & 66.67   & 65.06   & 64.81 \\ \hline
$\nicefrac{2}{255}$ & 27.22 & 27.31   & 31.36   & 16.57   & 27.43   & 28.25   & 27.92 \\ \hline
$\nicefrac{8}{255}$ & 18.46 & 20.05   & 26.09   & 10.41   & 21.15   & 22.50   & 22.00 \\ \hline
$\nicefrac{16}{255}$& 10.49 & 11.30   & 14.97   & 5.10    & 11.30   & 12.64   & 13.11 \\ \toprule
\end{tabular}
\end{table*}

\begin{table*}
\centering
\caption{\label{tab:vgg-cifar100-p3}\textbf{Search Space for VGG16 on CIFAR100 -- Third Layer.} Test set accuracies comparison with the standard VGG16 and his enhanced versions. The first and second convolutional layers are modified with $p_1=3$ and $p_2=1$ families respectively. The third Gabor layer number of families $p_3$ is explored. S stands for standard.}
\small
\begin{tabular}{c||c||c||c|c}
\multicolumn{5}{c}{\textbf{Accuracy}}\\\toprule
$\epsilon$          & S     & $p_1=3, p_2=1$ & \textbf{$p_3=2$} & $p_3=3$ \\ \midrule
$0$                 & 67.54 & 64.10          & 63.74   & 64.49 \\ \hline
$\nicefrac{2}{255}$ & 27.22 & 31.36          & 17.13   & 31.12 \\ \hline
$\nicefrac{8}{255}$ & 18.46 & 26.09          & 14.67   & 25.82 \\ \hline
$\nicefrac{16}{255}$& 10.49 & 14.97          & 9.20    & 15.40\\ \toprule
\end{tabular}
\end{table*}

\begin{table*}
\centering
\caption{\label{tab:wrn-svhn-p1}\textbf{Search Space for Wide-ResNet on SVHN -- First Layer.} Test set accuracies comparison with the standard Wide-ResNet and his enhanced versions. The first Gabor layer number of families $p_1$ is explored. S stands for standard.}
\small
\begin{tabular}{c||c||c|c|c|c|c}
\multicolumn{7}{c}{\textbf{Accuracy}}\\\toprule
$\epsilon$          & S     & $p_1=2$ & $p_1=3$ & $p_1=4$ & $p_1=5$ & $p_1=6$ \\ \midrule
$0$                 & 96.62 & 96.93   & 96.72   & 96.70   & 96.65   & 96.73   \\ \hline
$\nicefrac{2}{255}$ & 40.27 & 45.84   & 41.37   & 49.35   & 43.56   & 45.66   \\ \hline
$\nicefrac{8}{255}$ & 1.03  & 0.93    & 1.09    & 1.03    & 1.08    & 1.03    \\ \hline
$\nicefrac{16}{255}$& 1.32  & 1.11    & 1.32    & 1.42    & 1.38    & 1.28    \\ \toprule
\end{tabular}
\end{table*}

\begin{table*}
\centering
\caption{\label{tab:wrn-svhn-p2}\textbf{Search Space for Wide-ResNet on SVHN -- Second Layer.} Test set accuracies comparison with the standard Wide-ResNet and his enhanced versions. The first convolutional layer is modified with $p_1=4$ families. The second Gabor layer number of families $p_2$ is explored. S stands for standard.}
\small
\begin{tabular}{c||c||c||c|c|c|c}
\multicolumn{7}{c}{\textbf{Accuracy}}\\\toprule
$\epsilon$          & S     & $p_1=4$ & $p_2=1$ & $p_2=2$ & $p_2=3$ & $p_2=4$\\ \midrule
$0$                 & 96.62 & 96.70   & 96.67   & 96.71   & 96.60   & 96.71\\ \hline
$\nicefrac{2}{255}$ & 40.27 & 49.35   & 41.83   & 44.50   & 44.11   & 44.41\\ \hline
$\nicefrac{8}{255}$ & 1.03  & 1.03    & 1.01    & 0.99    & 0.98    & 1.03 \\ \hline
$\nicefrac{16}{255}$& 1.32  & 1.42    & 1.24    & 1.31    & 1.28    & 1.26 \\ \toprule
\end{tabular}
\end{table*}

\section{Flip rates}\label{sec:suppl_fliprates}
In our work, we assess the robustness of Deep Neural Networks (DNNs) through adversarial accuracies (reported in the main document) and flip rates. Next, we report the flip rates for the main experiments of the paper. Table \ref{tab:fliprates_supp} shows results on SVHN, CIFAR10 and CIFAR100, and Table \ref{tab:fliprates_mnist_supp} presents results on MNIST.

\begin{table*}[t]
\centering
\caption{\label{tab:fliprates_supp}\textbf{Flip rates comparison}. We compare Standard (S), Gabor-layered (G), and \textit{regularized} Gabor-layered (G+r) architectures. For each attack strength ($\epsilon$), the lowest flip rate is in \textbf{bold}; second-lowest is \unde{underlined}.}
\vspace{0.2cm}
\resizebox{\columnwidth}{!}{
\begin{tabular}{c c||c c c|c c c|c c c}
\toprule
\multicolumn{2}{c|}{$\epsilon$}                     &  \multicolumn{3}{c|}{$\nicefrac{2}{255}$}     &  \multicolumn{3}{c|}{$\nicefrac{8}{255}$}     & \multicolumn{3}{c}{$\nicefrac{16}{255}$}      \\\midrule
Dataset&Network & S             & G             & G+r           & S             & G             & G+r           & S             & G             & G+r           \\\midrule
SVHN&WRN        & 57.22         & \unde{48.13}  & \textbf{44.19}& 98.30         & \textbf{98.13}& \unde{98.17}  & 99.06         & \unde{99.02}  & \textbf{99.01}         \\
SVHN&VGG16      & 39.53         & \unde{34.11}  & \textbf{32.62}& 92.82         & \unde{83.83}  & \textbf{82.33}& 97.40         & \unde{91.56}  & \textbf{90.78}         \\
CIFAR10&VGG16   & 60.03         & \unde{56.42}  & \textbf{55.71}& 74.51         & \unde{67.83}  & \textbf{67.46}& 86.47         & \unde{80.84}  & \textbf{80.18}\\
CIFAR100&AN     & 56.88         & \unde{49.73}  & \textbf{49.15}& 81.38         & \textbf{72.82}& \unde{72.87}  & 87.77         & \textbf{83.00}& \unde{83.08}  \\
CIFAR100&WRN    & 82.05         & \unde{76.52}  & \textbf{74.72}& 93.82         & \unde{92.58}  & \textbf{92.25}& 96.94         & \unde{96.59}  & \textbf{96.35}         \\
CIFAR100&VGG16  & 57.05         & \textbf{50.94}& \unde{51.05}  & 77.94         & \unde{68.95}  & \textbf{68.45}& 90.48         & \unde{85.75}  & \textbf{85.16}
\\\toprule
\end{tabular}}
\end{table*}

\begin{table*}[t]
\centering
\caption{\label{tab:fliprates_mnist_supp}\textbf{Flip rates comparison on MNIST}. We compare Standard (S), Gabor-layered (G), and \textit{regularized} Gabor-layered (G+r) architectures on MNIST. For each attack strength ($\epsilon$), the lowest flip rate is in \textbf{bold}; second-lowest is \unde{underlined}.}
\vspace{0.2cm}
\resizebox{\columnwidth}{!}{
\begin{tabular}{c c||c c c|c c c|c c c}
\toprule
\multicolumn{2}{c|}{$\epsilon$}                     &  \multicolumn{3}{c|}{$0.1$}     &  \multicolumn{3}{c|}{$0.2$}     & \multicolumn{3}{c}{$0.3$}             \\\midrule
Dataset&Network & S             & G             & G+r           & S             & G             & G+r           & S             & G             & G+r           \\\midrule
MNIST&LeNet     & 19.53         & \unde{18.88}  & \textbf{11.05}& 95.47         & \unde{91.85}  & \textbf{77.27}& 99.70         & \textbf{99.24}& \unde{99.64}  \\\toprule
\end{tabular}}
\end{table*}

\section{ImageNet Results}\label{sec:suppl_imagenet}
We conduct adversarial attacks for $\epsilon \in\{\nicefrac{8}{255}, \nicefrac{16}{255}\}$. In Table \ref{tab:imagenet_results} we report the adversarial accuracies and flip rates for VGG16 \cite{vgg_simonyan} trained on ImageNet \cite{russakovsky2015imagenet}.

\begin{table*}
\centering
\caption{\label{tab:imagenet_results} \textbf{Adversarial accuracy and flip rate comparison for VGG16 on ImageNet.} We compare Standard (S) and Gabor-layered (G) architectures. For each attack strength ($\epsilon$), the best performance is in \textbf{bold}.}
\small
\begin{tabular}{c||p{1.3cm} p{1.3cm}||p{1.3cm} p{1.3cm}}
& \multicolumn{2}{c||}{\textbf{Adv. Accuracy}} & \multicolumn{2}{c}{\textbf{Flip Rate}}\\\toprule
$\epsilon$          & S             & G     & S     & G             \\ \midrule
$0$                 & \textbf{71.20}& 68.90 & -     & -             \\ \hline
$\nicefrac{8}{255}$ & \textbf{2.95} & 2.24  & 95.07 & \textbf{94.46}\\ \hline
$\nicefrac{16}{255}$& \textbf{3.33} & 3.15  & 97.37 & \textbf{96.68}\\ \toprule
\end{tabular}

\end{table*}

\section{Singular Values}\label{sec:suppl_singvals}
We report the distribution of singular values of the filters of up to the first three convolutional layers of several Convolutional Neural Network (CNN) architectures (LeNet~\cite{lecun_lenet}, AlexNet~\cite{alexnet}, WideResNet~\cite{wideresnet_zagoruyko}, and VGG16~\cite{vgg_simonyan}) trained in various datasets (MNIST~\cite{lecun_mnist}, CIFAR10, CIFAR100~\cite{Cifars}, and ImageNet~\cite{russakovsky2015imagenet}). The singular values of the layers are computed following~\cite{sedghi2018the,bibi2018deep}. The largest singular value of each layer's filter corresponds to the filters' Lipschitz constant~\cite{bibi2018deep}. 

In each histogram plot we show the distribution of singular values of (1) the standard architecture, in blue, and (2) the Gabor-layered architecture, in orange. For some experiments, we also show the distribution of singular values of the Gabor-layered architecture \textit{with} regularization, in green. 

For visualization purposes, we set the upper x-limit of each histogram plot to the $95$th percentile of the distribution with the largest maximum value.

Next, we list the dataset-architecture pairs, and the Figures in which its distributions are shown:
\begin{itemize}
    \item \textbf{MNIST-LeNet.} Figure~\ref{fig:sing_vals_mnist_lenet_conv1}.
    
    \item \textbf{CIFAR100-AlexNet.} Figure~\ref{fig:sing_vals_cifar100_alexnet_conv1}.
    
    \item \textbf{CIFAR10-VGG16.} Figures~\ref{fig:sing_vals_cifar10_vgg16_conv1}~-~\ref{fig:sing_vals_cifar10_vgg16_conv3}.
    
    \item \textbf{CIFAR100-VGG16.} Figures~\ref{fig:sing_vals_cifar100_vgg16_conv1}~-~\ref{fig:sing_vals_cifar100_vgg16_conv3}.
    
    \item \textbf{ImageNet-VGG16.} Figures~\ref{fig:sing_vals_imagenet_vgg16_conv1}~-~\ref{fig:sing_vals_imagenet_vgg16_conv3}.

\end{itemize}

In most cases, the distribution of singular values of the Gabor-layered version of the networks tends to be around smaller values, usually with high peaks between $0$ and $0.5$, than that of the standard network. 

In terms of the Lipschitz constant of the layers, in most cases we observe that the Gabor-layered versions of the layers have lower Lipschitz constants, and that applying regularization results in even lower Lipschitz constants.

\section{Filter visualizations}\label{sec:suppl_filtervis}
We report visualizations of the filters of the first convolutional layer of some of the architectures we experimented with. Figures~\ref{fig:lenet_std_kernels} through~\ref{fig:vgg16_imagenet_gabor_kernels} depict the filters. For the standard convolutional layers, we visualize each filter as an RGB image, where each \say{channel} of the filter is scaled to be between $0$ and $1$. 

For the Gabor layers, each filter has $1$ channel and, hence, we visualize it as is. We show the filters that share the Gabor function $G_\theta$ in the same row; while each column corresponds to one of the $8$ $\alpha$-scaled rotations of the filter.

We report the filters for:
\begin{itemize}
    \item \textbf{MNIST-LeNet.} Figures~\ref{fig:lenet_std_kernels}~-~\ref{fig:lenet_gabor_reg_kernels}.
    
    \item \textbf{CIFAR100-AlexNet.} Figures~\ref{fig:alexnet_std_kernels}~-~\ref{fig:alexnet_gabor_reg_kernels}.
    
    \item \textbf{ImageNet-VGG16.} Figures~\ref{fig:vgg16_imagenet_std_kernels}~-~\ref{fig:vgg16_imagenet_gabor_kernels}.
\end{itemize}

We note that, for LeNet on MNIST, the Gabor-layered version, without regularization, already has filters that strongly resemble a Dirac-delta function and that, as reported in the paper, this network already shows improvements in terms of robustness. Furthermore, when applying regularization, we observe that all filters in the layer become Dirac-delta functions (see Figure~\ref{fig:lenet_gabor_reg_kernels}). Again, as reported in the paper, regularization showed large gains in robustness.

The filters from AlexNet are useful for visualizing the modeling capabilities of Gabor functions, as shown in Figure~\ref{fig:alexnet_gabor_kernels}, where we observe blob-like patterns, and also oriented and scaled edges and bars. These patterns, while simpler than those of the standard convolutional layers (see Figure~\ref{fig:alexnet_std_kernels}), provide on-pair accuracy with such layers, while also providing gains in robustness. For the case of AlexNet, however, we observe that regularization has virtually no impact in the form of the filters that are learnt (see Figure~\ref{fig:alexnet_gabor_reg_kernels}).

In Figure~\ref{fig:vgg16_imagenet_std_kernels} we report the filters learnt by the standard version of VGG16 on ImageNet. Patterns are, however, not straightforward to visualize in these filters. By construction, the filters learnt by the Gabor-layered version of VGG16 when fine-tuned on ImageNet are more visually-appealing, as shown in Figure~\ref{fig:vgg16_imagenet_gabor_kernels}. Note, also, that some Gabor filters have strong similarities between one another.

\newpage

\begin{figure*}[!h]
\centering
\includegraphics[width=0.85\textwidth]{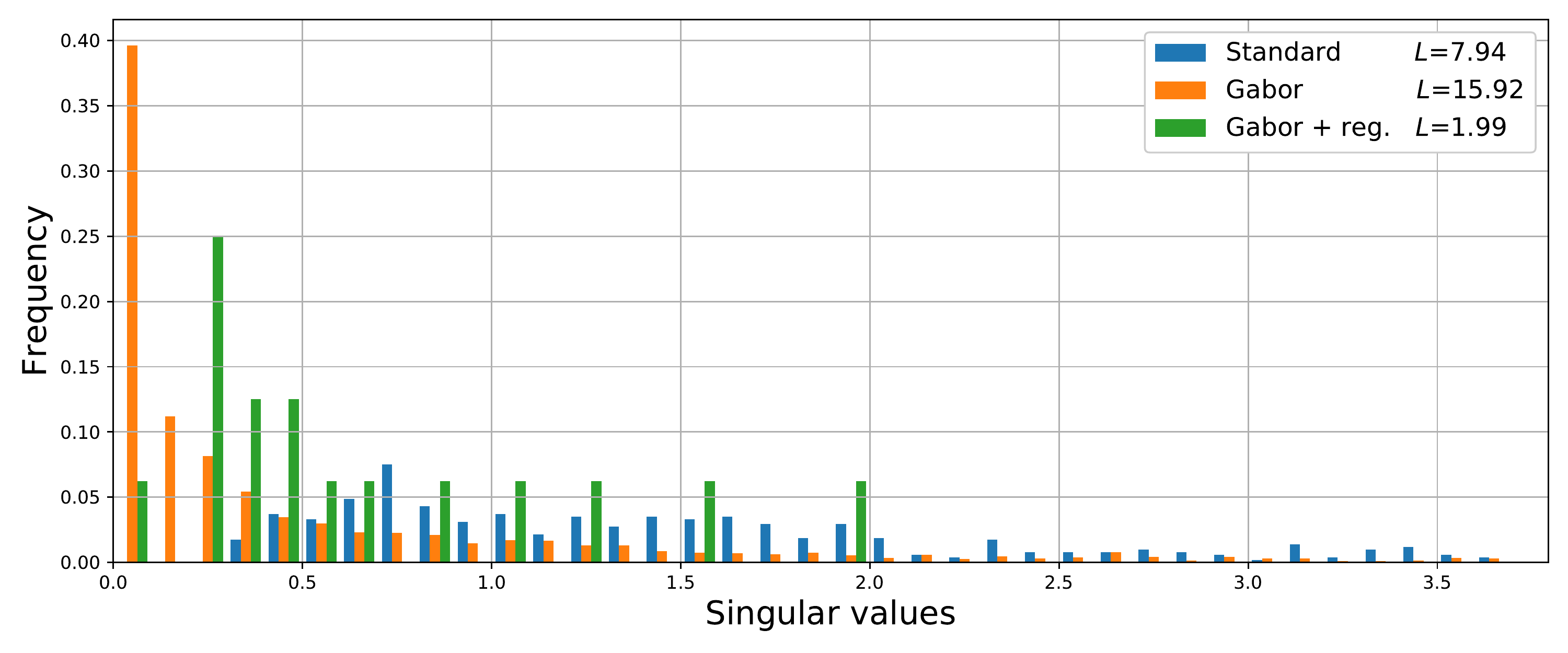}
\caption{\textbf{Distribution of singular values for the first layer of LeNet trained on MNIST.} The legend shows the largest singular value, \ie the Lipschitz constant of the layer.}
\label{fig:sing_vals_mnist_lenet_conv1}
\end{figure*}

\begin{figure*}[!h]
\centering
\includegraphics[width=0.85\textwidth]{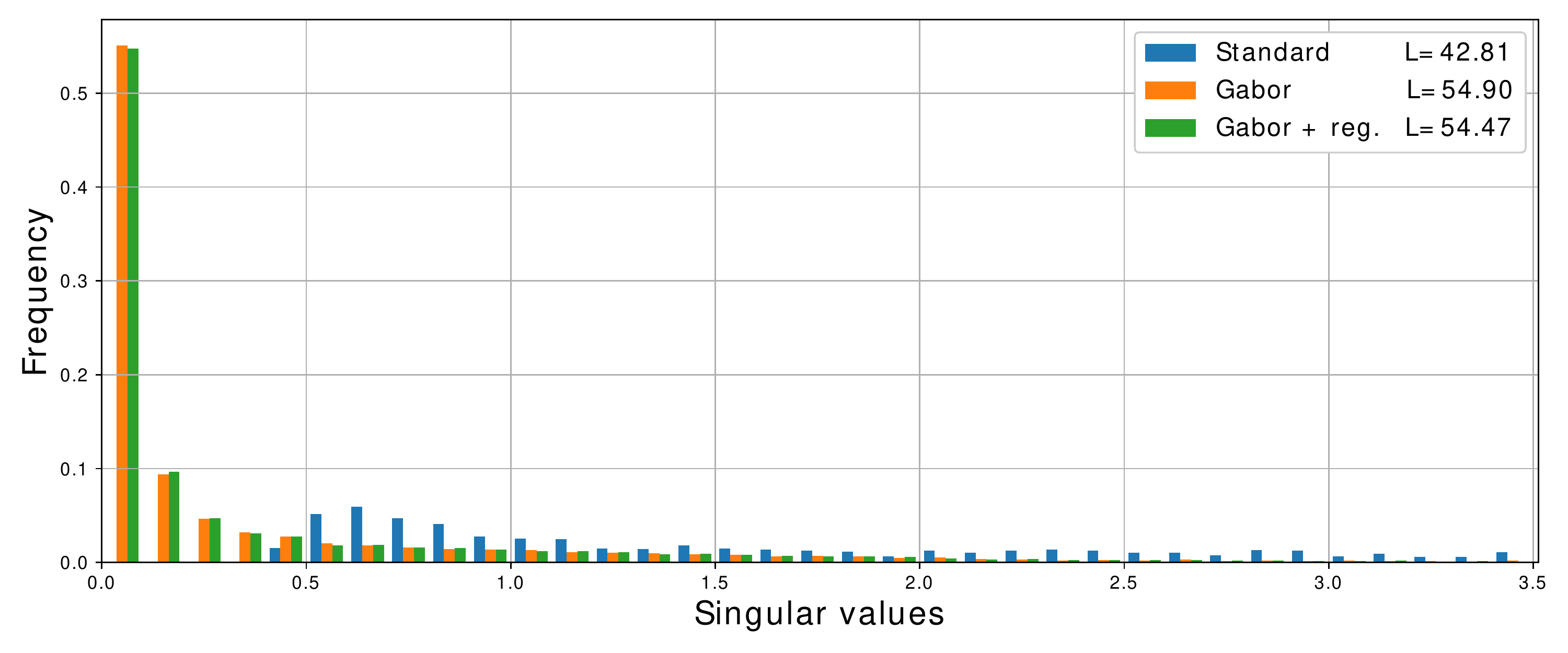}
\caption{\textbf{Distribution of singular values for the first layer of AlexNet trained on CIFAR100.} The legend shows the largest singular value, \ie the Lipschitz constant of the layer.}
\label{fig:sing_vals_cifar100_alexnet_conv1}
\end{figure*}

\begin{figure*}[!h]
\centering
\includegraphics[width=0.85\textwidth]{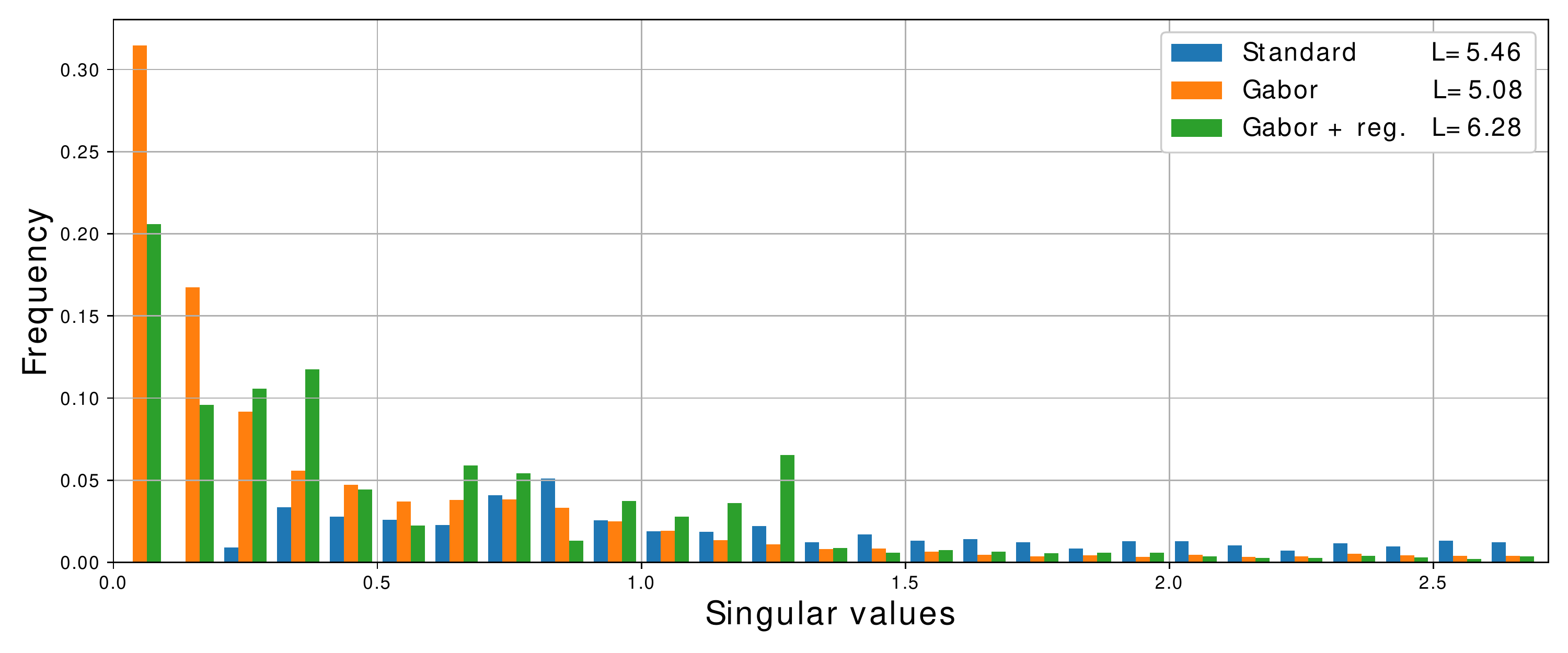}
\caption{\textbf{Distribution of singular values for the first layer of VGG16 trained on CIFAR10.} The legend shows the largest singular value, \ie the Lipschitz constant of the layer.}
\label{fig:sing_vals_cifar10_vgg16_conv1}
\end{figure*}

\begin{figure*}[!h]
\centering
\includegraphics[width=0.85\textwidth]{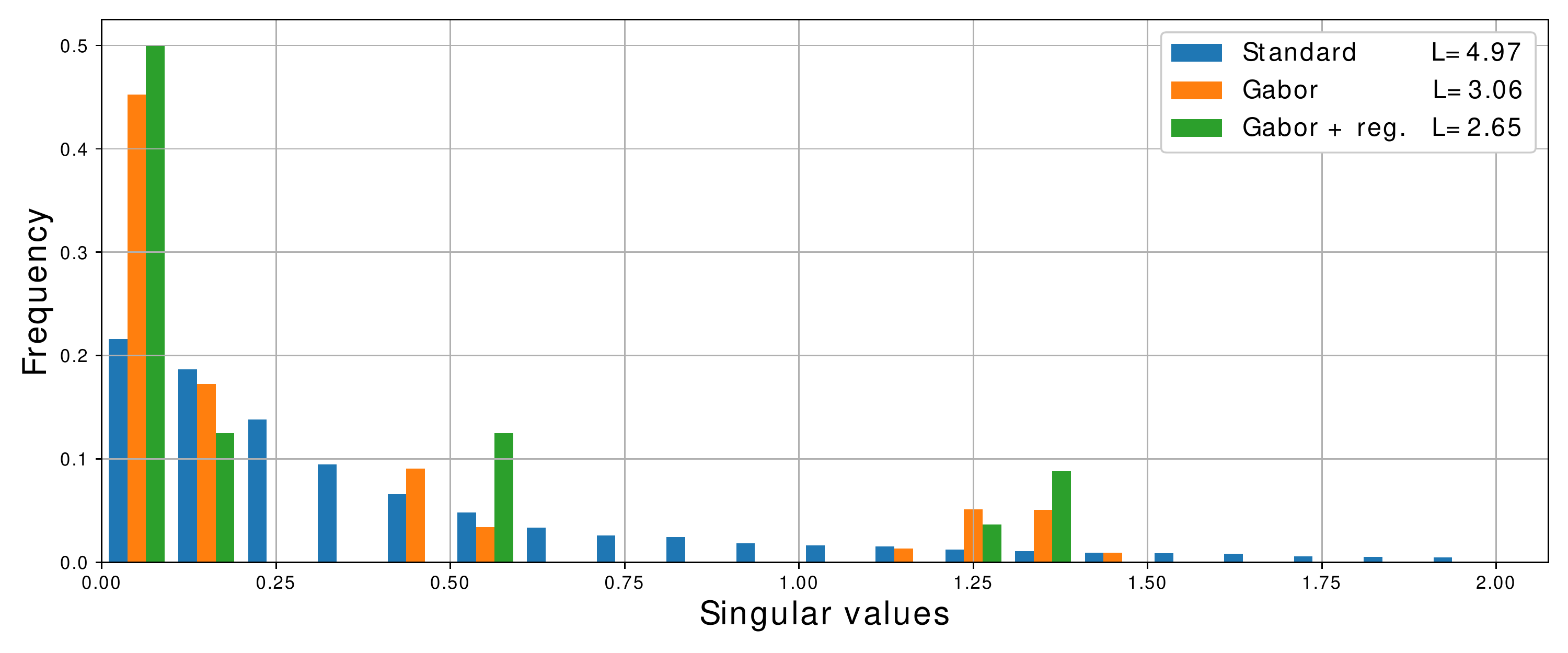}
\caption{\textbf{Distribution of singular values for the second layer of VGG16 trained on CIFAR10.} The legend shows the largest singular value, \ie the Lipschitz constant of the layer.}
\label{fig:sing_vals_cifar10_vgg16_conv2}
\end{figure*}

\begin{figure*}[!h]
\centering
\includegraphics[width=0.85\textwidth]{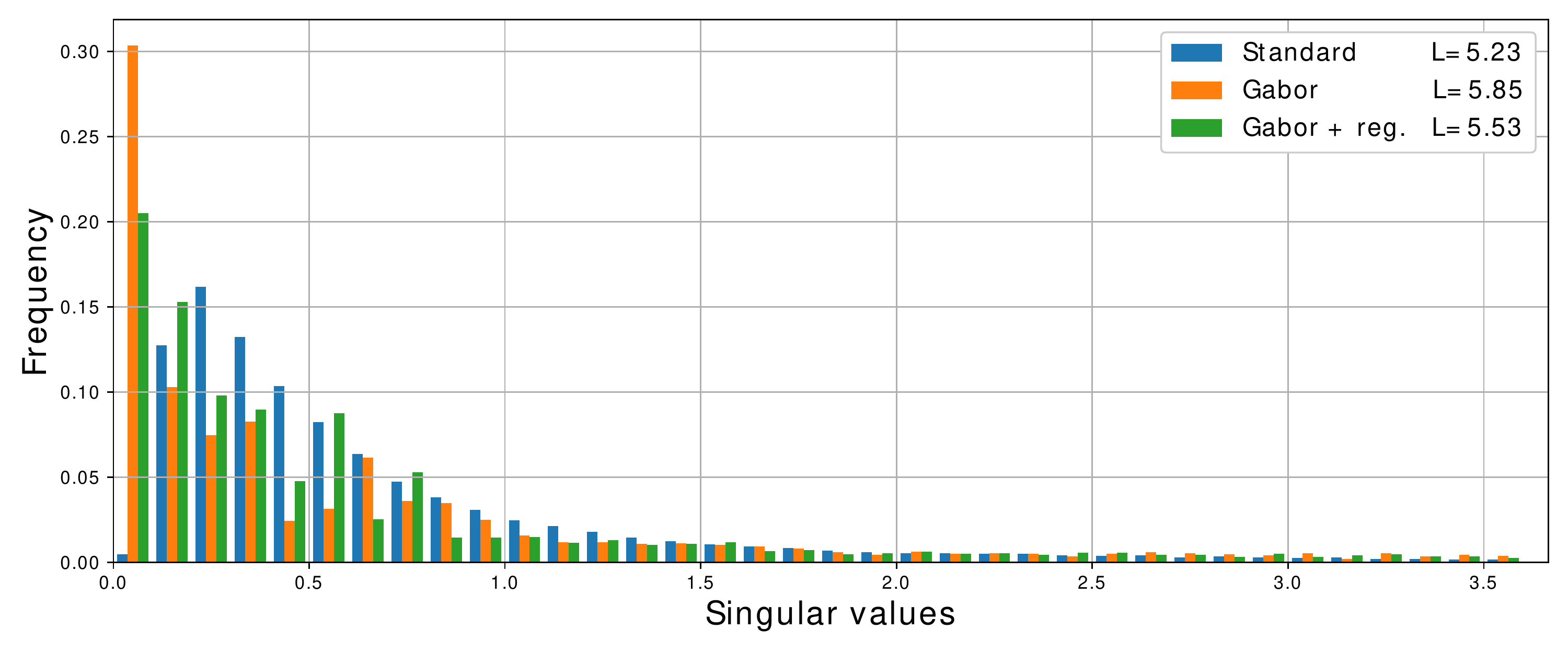}
\caption{\textbf{Distribution of singular values for the third layer of VGG16 trained on CIFAR10.} The legend shows the largest singular value, \ie the Lipschitz constant of the layer.}
\label{fig:sing_vals_cifar10_vgg16_conv3}
\end{figure*}

\begin{figure*}[!h]
\centering
\includegraphics[width=0.85\textwidth]{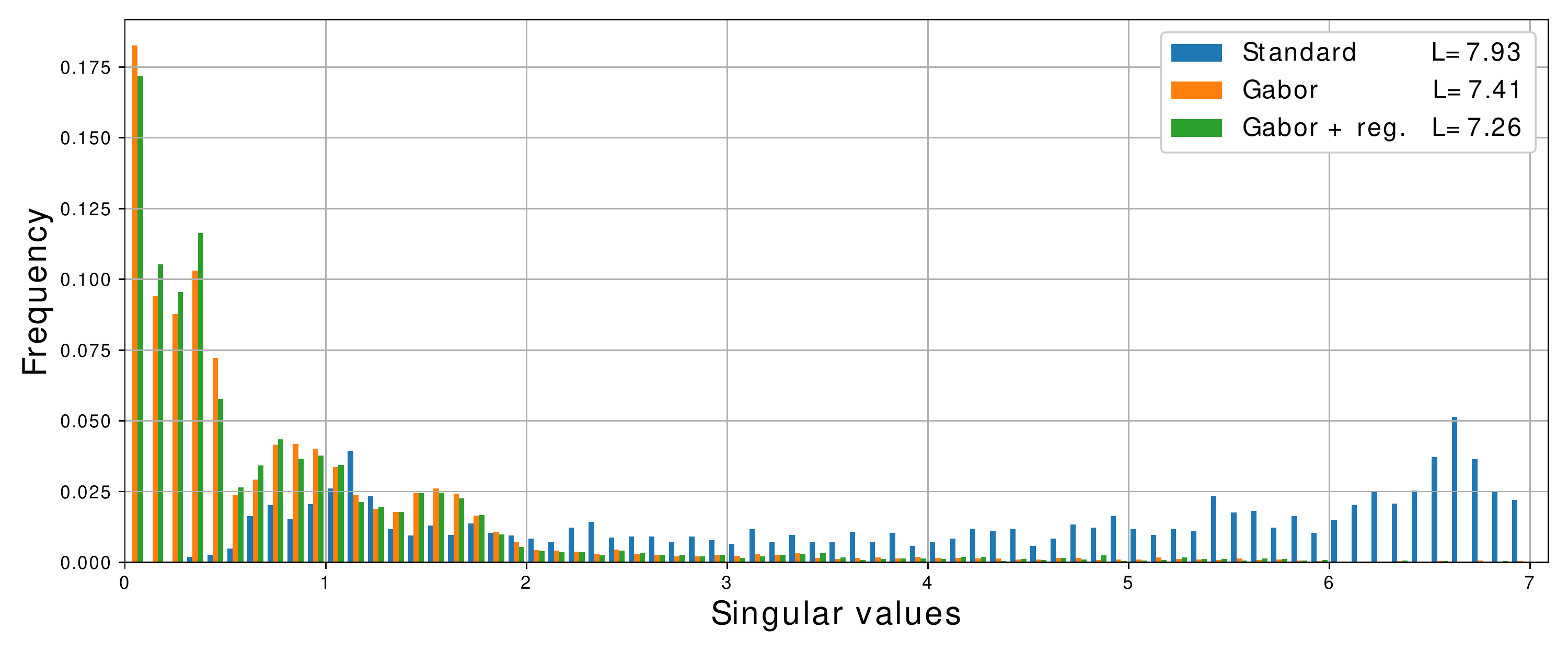}
\caption{\textbf{Distribution of singular values for the first layer of VGG16 trained on CIFAR100.} The legend shows the largest singular value, \ie the Lipschitz constant of the layer.}
\label{fig:sing_vals_cifar100_vgg16_conv1}
\end{figure*}

\begin{figure*}[!h]
\centering
\includegraphics[width=0.85\textwidth]{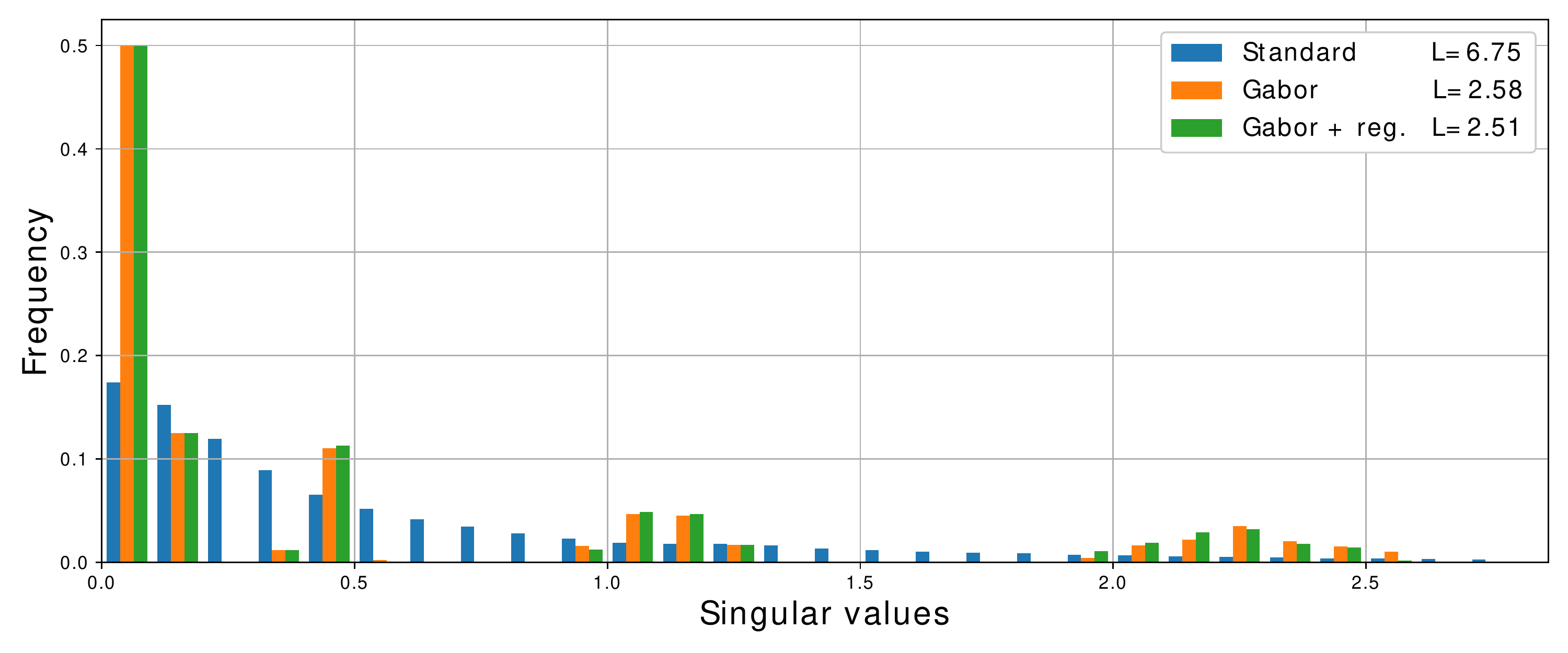}
\caption{\textbf{Distribution of singular values for the second layer of VGG16 trained on CIFAR100.} The legend shows the largest singular value, \ie the Lipschitz constant of the layer.}
\label{fig:sing_vals_cifar100_vgg16_conv2}
\end{figure*}

\begin{figure*}[!h]
\centering
\includegraphics[width=0.85\textwidth]{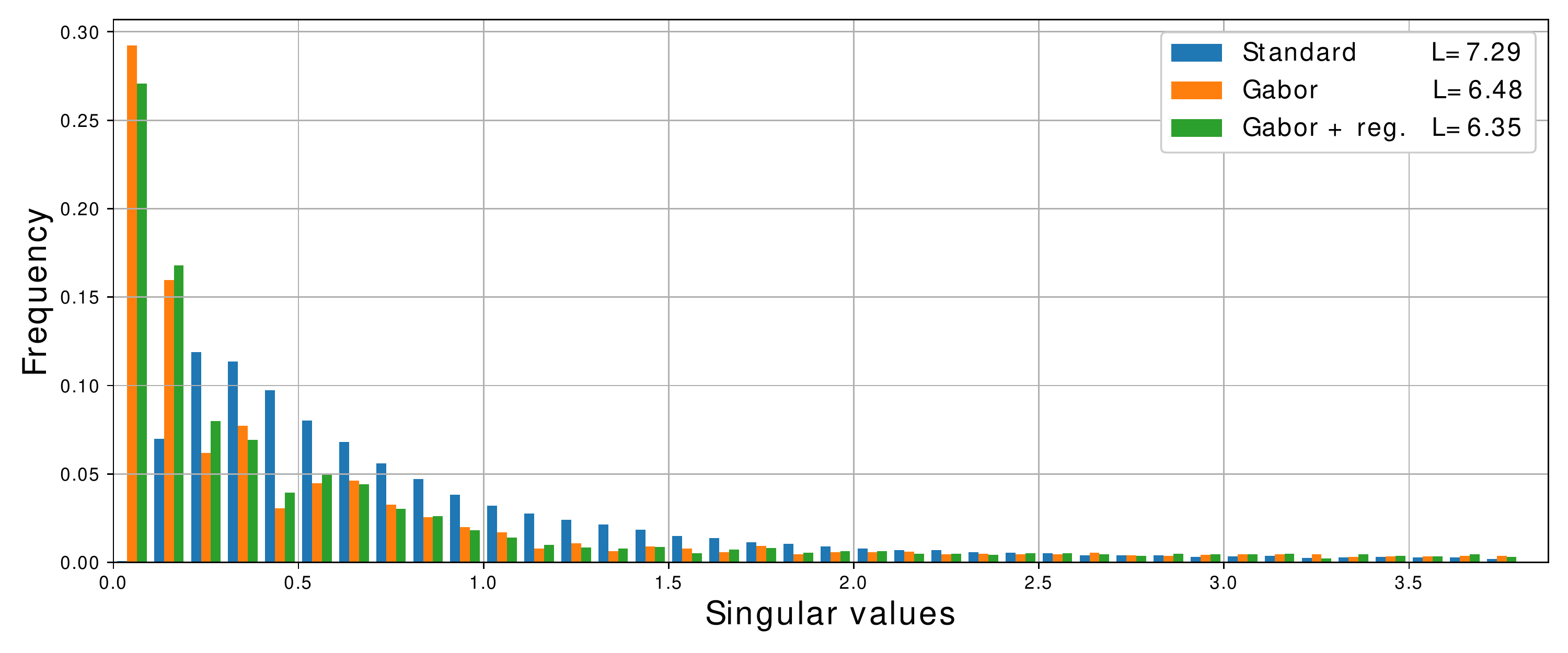}
\caption{\textbf{Distribution of singular values for the third layer of VGG16 trained on CIFAR100.} The legend shows the largest singular value, \ie the Lipschitz constant of the layer.}
\label{fig:sing_vals_cifar100_vgg16_conv3}
\end{figure*}

\begin{figure*}[!h]
\centering
\includegraphics[width=0.85\textwidth]{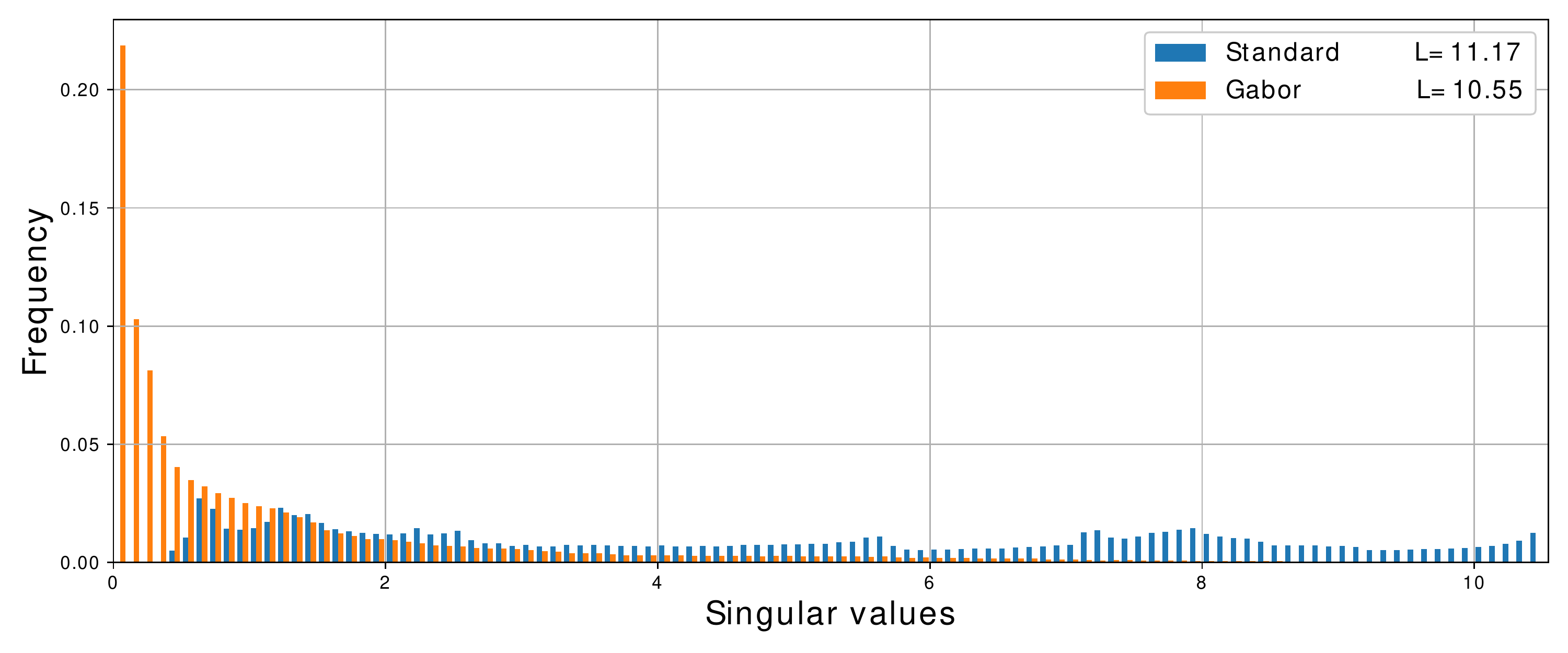}
\caption{\textbf{Distribution of singular values for the first layer of VGG16 trained on ImageNet.} The legend shows the largest singular value, \ie the Lipschitz constant of the layer. Note that the Gabor-layered version was fine-tuned, starting from ImageNet-pretrained weights, due to computational constraints.}
\label{fig:sing_vals_imagenet_vgg16_conv1}
\end{figure*}

\begin{figure*}[!h]
\centering
\includegraphics[width=0.85\textwidth]{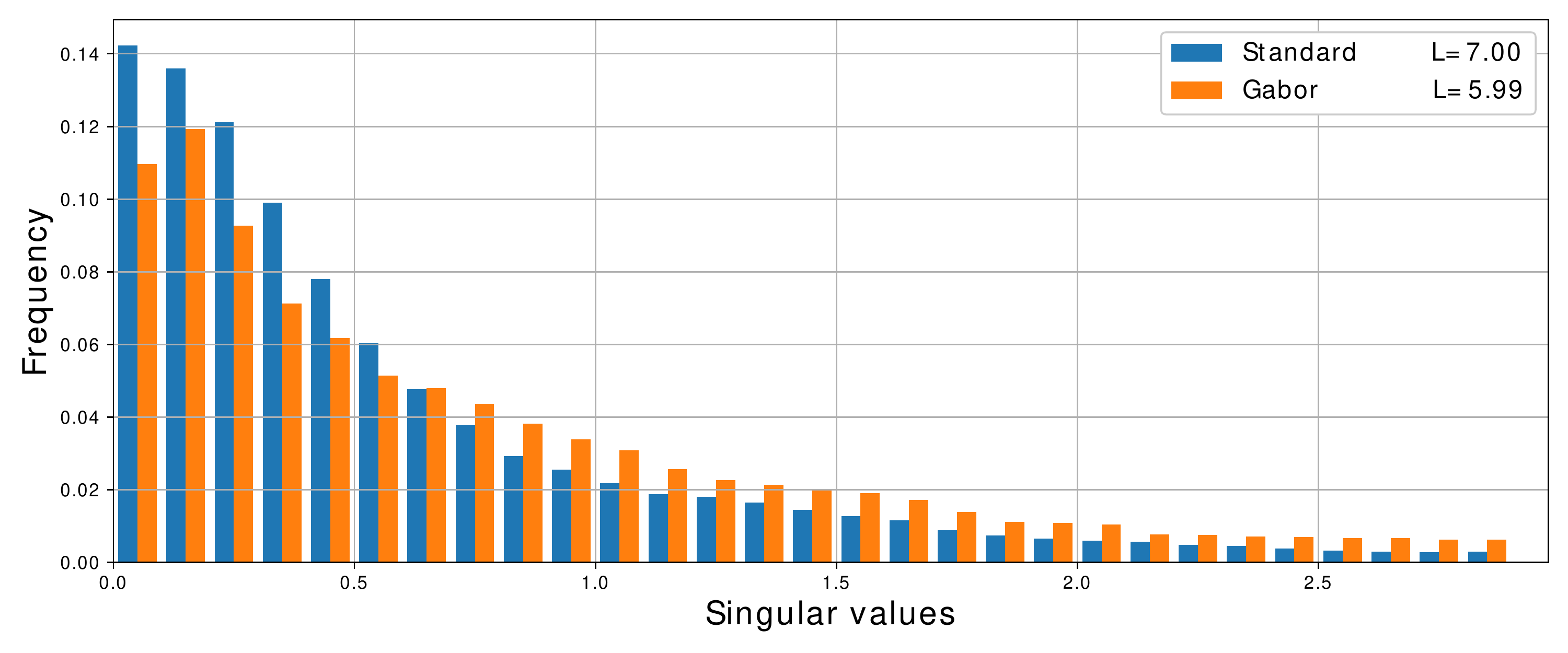}
\caption{\textbf{Distribution of singular values for the second layer of VGG16 trained on ImageNet.} The legend shows the largest singular value, \ie the Lipschitz constant of the layer. Note that the Gabor-layered version was fine-tuned, starting from ImageNet-pretrained weights, due to computational constraints.}
\label{fig:sing_vals_imagenet_vgg16_conv2}
\end{figure*}

\begin{figure*}[!h]
\centering
\includegraphics[width=0.85\textwidth]{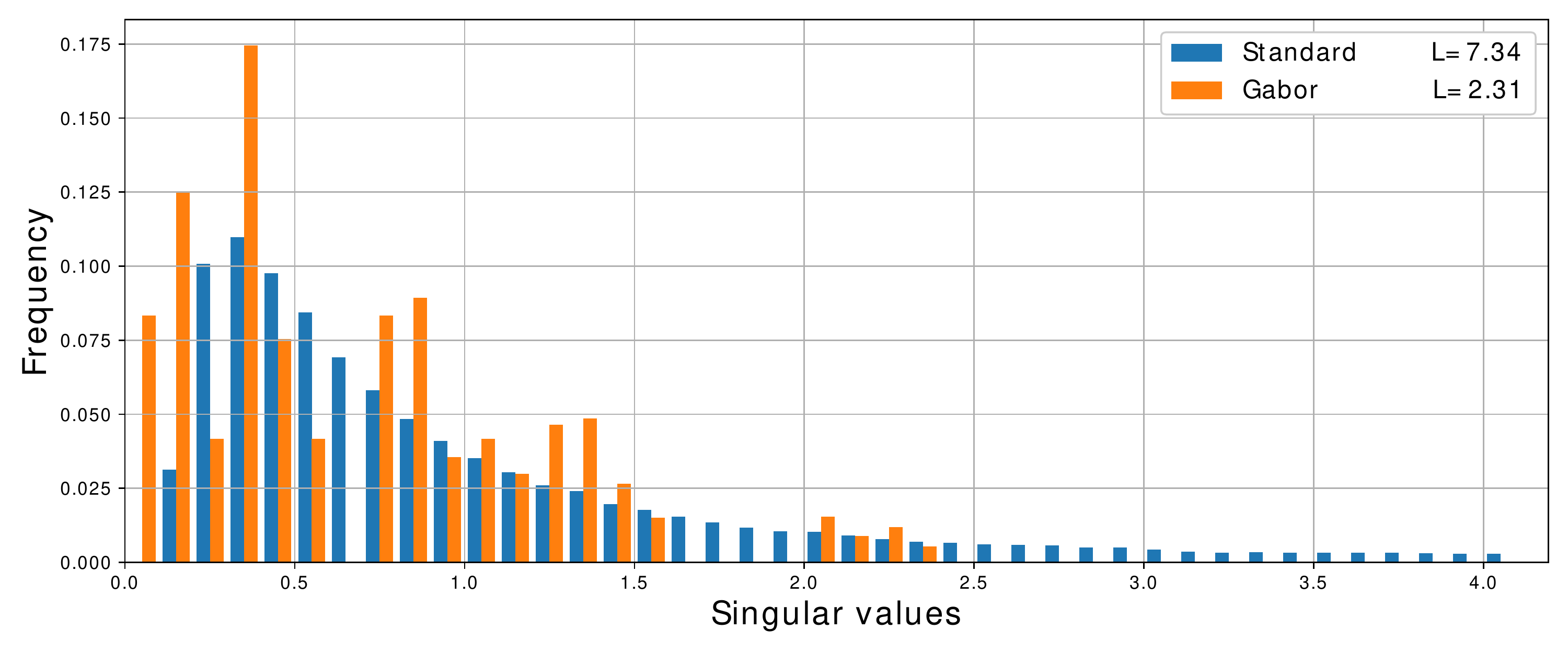}
\caption{\textbf{Distribution of singular values for the third layer of VGG16 trained on ImageNet.} The legend shows the largest singular value, \ie the Lipschitz constant of the layer. Note that the Gabor-layered version was fine-tuned, starting from ImageNet-pretrained weights, due to computational constraints.}
\label{fig:sing_vals_imagenet_vgg16_conv3}
\end{figure*}

\pagebreak

\begin{figure*}[!h]
\centering
\includegraphics[width=0.75\textwidth]{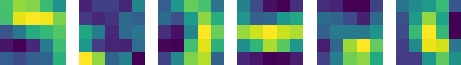}
\caption{\textbf{Standard LeNet-MNIST filters from the first convolutional layer.} The $6$ grayscale filters are visualized.}
\label{fig:lenet_std_kernels}
\end{figure*}

\begin{figure*}[!h]
\centering
\includegraphics[width=0.75\textwidth]{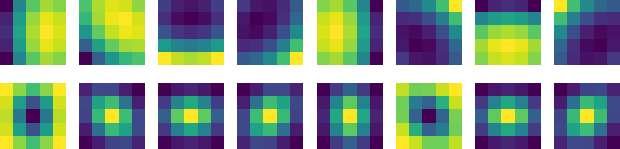}
\caption{\textbf{Gabor-layered LeNet-MNIST filters from the first convolutional layer.} Each of the rows in the figure corresponds to each of the $7$ families of filters of the layer. Each column is one of the $8$ $\alpha$-scaled rotations of the filter. Note that the filters in the second row resemble a Dirac-delta function.}
\label{fig:lenet_gabor_kernels}
\end{figure*}

\begin{figure*}[!h]
\centering
\includegraphics[width=0.75\textwidth]{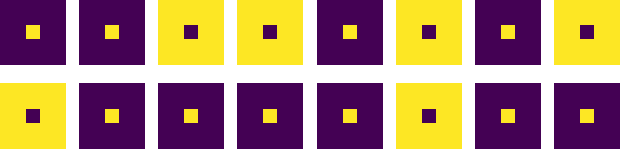}
\caption{\textbf{Gabor-layered LeNet-MNIST filters from the first convolutional layer with regularization.} Each of the rows in the figure corresponds to each of the $7$ families of filters of the layer. Each column is one of the $8$ $\alpha$-scaled rotations of the filter. Note that regularization enforces the filters to be Dirac-deltas and, as reported in the paper, this change results in large gains in robustness.}
\label{fig:lenet_gabor_reg_kernels}
\end{figure*}

\newpage

\begin{figure*}[!h]
\centering
\includegraphics[width=0.95\textwidth]{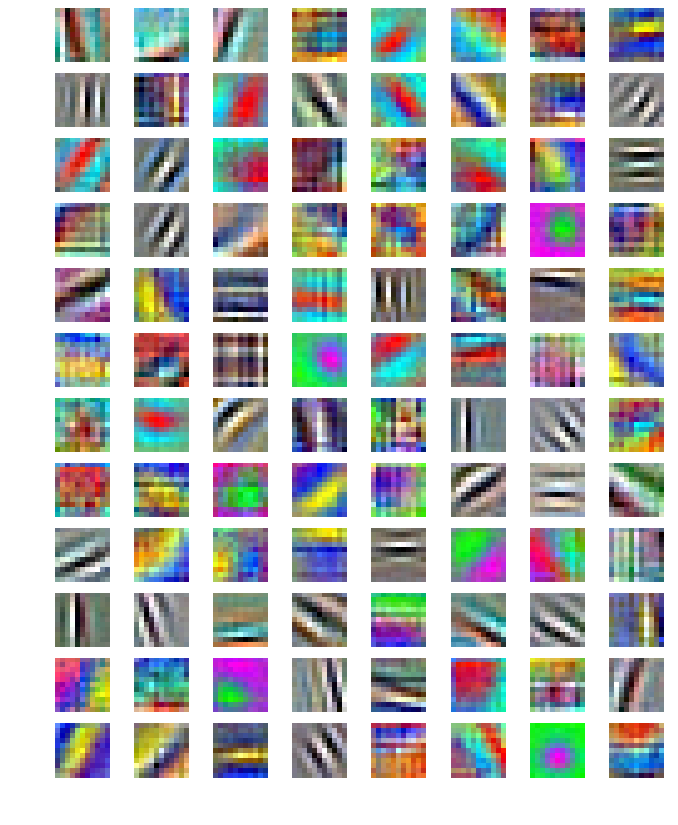}
\caption{\textbf{Standard AlexNet-CIFAR100 filters from the first convolutional layer.} The $96$ filters are visualized by concatenating the RGB channels and mapping values from $0$ to $1$.}
\label{fig:alexnet_std_kernels}
\end{figure*}

\begin{figure*}[!h]
\centering
\includegraphics[width=0.7\textwidth]{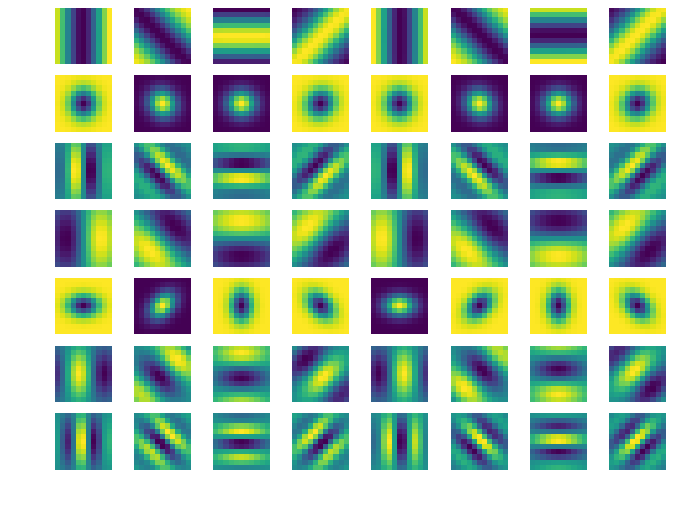}
\caption{\textbf{Gabor-layered AlexNet-CIFAR100 filters from the first convolutional layer.} Each of the rows in the figure corresponds to each of the $7$ families of filters of the layer. Each column is one of the $8$ $\alpha$-scaled rotations of the filter.}
\label{fig:alexnet_gabor_kernels}
\end{figure*}

\begin{figure*}[!h]
\centering
\includegraphics[width=0.7\textwidth]{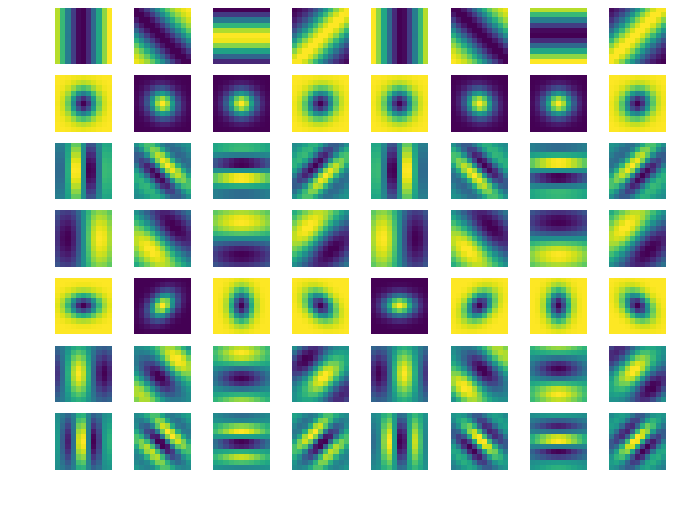}
\caption{\textbf{Gabor-layered AlexNet-CIFAR100 filters from the first convolutional layer with regularization.} Each of the rows in the figure corresponds to each of the $7$ families of filters of the layer. Each column is one of the $8$ $\alpha$-scaled rotations of the filter.}
\label{fig:alexnet_gabor_reg_kernels}
\end{figure*}

\newpage

\begin{figure*}[!h]
\centering
\includegraphics[width=0.7\textwidth]{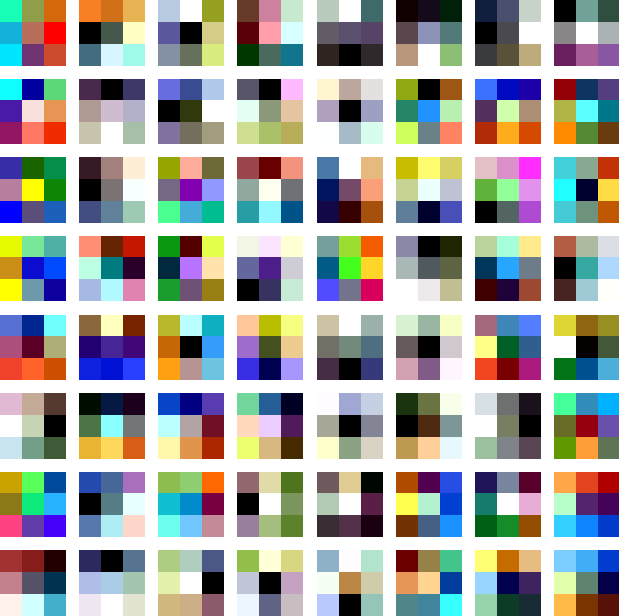}
\caption{\textbf{Standard VGG16-ImageNet filters from the first convolutional layer.} The $64$ filters are visualized by concatenating the RGB channels and mapping values from $0$ to $1$.}
\label{fig:vgg16_imagenet_std_kernels}
\end{figure*}

\begin{figure*}[!h]
\centering
\includegraphics[width=0.45\textwidth]{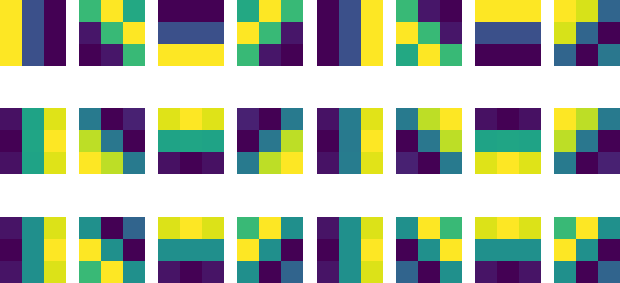}
\caption{\textbf{Gabor-layered VGG16-ImageNet filters from the first convolutional layer without regularization.} Each of the rows in the figure corresponds to each of the $3$ families of filters of the layer. Each column is one of the $8$ $\alpha$-scaled rotations of the filter.}
\label{fig:vgg16_imagenet_gabor_kernels}
\end{figure*}